\documentclass{article}

\usepackage[accepted]{icml2025}
\usepackage{fancyhdr}

\usepackage[utf8]{inputenc} % allow utf-8 input
\usepackage[T1]{fontenc}    % use 8-bit T1 fonts
\usepackage{hyperref}       % hyperlinks
\usepackage{url}            % simple URL typesetting
\usepackage{booktabs}       % professional-quality tables
\usepackage{amsfonts}       % blackboard math symbols
\usepackage{nicefrac}       % compact symbols for 1/2, etc.
\usepackage{microtype}      % microtypography
\usepackage{xcolor}         % colors

\usepackage{amsmath}
\usepackage{amssymb}
\usepackage{graphicx}

\usepackage{verbatim}
\usepackage{caption,subcaption}
\usepackage[overload]{empheq}
\usepackage[capitalize,noabbrev]{cleveref} 
\usepackage{amsthm}
\usepackage{epsf, epsfig, hyperref, mathtools,amsthm,bbm,xcolor,pifont}
\usepackage{algorithm}
\usepackage{bm}

\numberwithin{equation}{section}

\newtheorem{remark}{Remark}[section]

\newtheorem{theorem}{Theorem}[section]
\newtheorem{definition}{Definition}[section]
\newtheorem{proposition}{Proposition}[section]

\def \E{\mathbb{E}}

\def \N{\mathbb{N}}
\def \P{\mathbb{P}}

\def \R{\mathbb{R}}
\def \Pc{\mathcal{P}}

\newcommand{\mud}{\mu_{\operatorname{d}}}

\DeclareMathOperator{\Lip}{Lip}

\usepackage[textsize=tiny]{todonotes}

\icmltitlerunning{A Differential Equation Approach for WGANs}

\begin{document}

\twocolumn[
\icmltitle{A Differential Equation Approach for Wasserstein GANs and Beyond}

\icmlsetsymbol{equal}{*}

\begin{icmlauthorlist}
\icmlauthor{Zachariah Malik}{yyy}
\icmlauthor{Yu-Jui Huang}{yyy}
%\icmlauthor{}{sch}
%\icmlauthor{}{sch}
\end{icmlauthorlist}

\icmlaffiliation{yyy}{Department of Applied Mathematics, University of Colorado, Boulder, CO 80309, United States}

\icmlcorrespondingauthor{Zachariah Malik}{zachariah.malik@colorado.edu}
\icmlcorrespondingauthor{Yu-Jui Huang}{yujui.huang@colorado.edu}

\vskip 0.3in
]

\printAffiliationsAndNotice{}
%\printAffiliationsAndNotice{\icmlEqualContribution}

\begin{abstract}
This paper proposes a new theoretical lens to view Wasserstein generative adversarial networks (WGANs). To minimize the Wasserstein-1 distance between the true data distribution and our estimate of it, we derive a distribution-dependent ordinary differential equation (ODE) which represents the gradient flow of the Wasserstein-1 loss, and show that a forward Euler discretization of the ODE converges. This inspires a new class of generative models that naturally integrates persistent training (which we call W1-FE). When persistent training is turned off, we prove that W1-FE reduces to WGAN. When we intensify persistent training, W1-FE is shown to outperform WGAN in training experiments from low to high dimensions, in terms of both convergence speed and training results. Intriguingly, one can reap the benefits only when persistent training is carefully integrated through our ODE perspective. As demonstrated numerically, a naive inclusion of persistent training in WGAN (without relying on our ODE framework) can significantly worsen training results.
\end{abstract}

\section{Introduction}
%Generative modeling, i.e., the task of generating data given some prespecified samples, is a fundamental problem in machine learning and artificial intelligence. Wasserstein generative adversarial networks (WGANs), as first introduced in the pivotal work \citet{Arjovsky17}, are a powerful class of models that seek to solve this problem. WGANs use the well-known adversarial training technique, wherein they train a generator to produce samples and a critic to discriminate between the generated samples and true data. While significant follow-up research into WGANs has focused on improving critic training \citep{Petzka18} \citep{Gulrajani17}, we are not aware of many breakthroughs in improving generator training.
Recently, \cite{Huang23} have shown that the original generative adversarial network (GAN) algorithm \citep{Goodfellow14} follows the dynamics of an ordinary differential equation (ODE) which represents the gradient flow induced by the Jensen-Shannon divergence \citep[Proposition 8]{Huang23}. \citet{Huang24} characterize the gradient flow induced by the Wasserstein-$2$ distance and propose a class of generative models, called W2-FE, that follows the dynamics of the corresponding ODE \citep[Algorithm 1]{Huang24}. Notably, while W2-FE covers W2-GAN in \citet{Leygonie19} as a special case, it allows for a modification to the generator training that may lead to substantial numerical improvements over W2-GAN \citep[Section 7]{Huang24}.

This paper is motivated by the question: can we obtain analogous results for Wasserstein generative adversarial networks (WGANs)? As a mainstream class of GAN algorithms known for its enhanced stability, WGAN \citep{Arjovsky17, Gulrajani17, Petzka18} attempts to generate samples whose distribution is similar to the true data distribution (measured by the Wasserstein-1 distance) using adversarial training motivated by \cite{Goodfellow14}. It is natural to expect that by slightly modifying the arguments in \citet{Huang24}, % under the Wasserstein-$2$ distance, 
one can similarly derive an ODE that represents the gradient flow induced by the Wasserstein-1 distance and correspondingly propose a new class of generative models that covers and potentially improves WGAN. 

Mathematical challenges abound, however, under the Wasserstein-1 distance. First and foremost, while subdifferential calculus is full-fledged under the Wasserstein-$p$ distance for all $p>1$ \citep[Chapter 10]{Ambrosio08}, the same construction breaks down when $p = 1$. In particular, the ``Fr\'{e}chet differential'' \citep[Corollary 10.2.7]{Ambrosio08} and the ``Wasserstein gradient'' \citep[Definition 5.62]{CD-book-18-I} are no longer well-defined for $p=1$. It is not even clear how ``gradient'' should be defined under the Wasserstein-$1$ distance. Moreover, when showing that a discretization of the gradient-flow ODE converges (as the time step tends to 0), \citet{Huang24} rely on an interpolation result from optimal transport: the most cost-efficient path to move one probability measure to another (measured by the Wasserstein-$p$ distance for any $p>1$) can be interpolated recursively into smaller segments between intermediate measures \citep[Lemma 7.2.1]{Ambrosio08}, where each segment is most cost-efficient in itself and represents one time step in the ODE discretization. Such an interpolation result, again, fails when $p=1$. 

To overcome these challenges, we first recall the ``linear functional derivative'' from the mean field game literature \citep{CD-book-18-I} and observe that the ``Euclidean gradient of a linear functional derivative'' can generally serve as the proper gradient notion in the space of probability measures, independently of any subdifferential calculus; see the discussion below Definition~\ref{def:LFD}. Our first main theoretical result, Proposition~\ref{Prop: LFD of W1 distance}, shows that the linear function derivative of the Wasserstein-1 distance exists and coincides with the Kantorovich potential in optimal transport (i.e., a 1-Lipschitz function that maximizes the duality formula of the Wasserstein-1 distance; see Definition~\ref{def:KP}). The gradient flow induced by the Wasserstein-1 distance thus takes the form of an ODE that evolves along the negative Euclidean gradient of the Kantorovich potential (i.e., \eqref{Eq: ODE} below). A  forward Euler discretization can be correspondingly devised (i.e., \eqref{Discrete Time W1 Transport Process} below). As mentioned above, the interpolation result in \citet{Ambrosio08} can no longer be used to show convergence of the discretization.  Our second main theoretical result, Theorem~\ref{Th: Convergence of interpolating measure curve}, instead relies on the uniform boundedness of the Euclidean gradient of the Kantorovich potential (see Remark~\ref{rem:nabla phi}), which is unique to the $p=1$ case. Such boundedness allows us to show appropriate compactness and equicontinuity of the flow of measures induced by the discretization, so that a refined Arzela-Ascoli argument can be applied to give the convergence of the discretization; see the discussion below Theorem~\ref{Th: Convergence of interpolating measure curve}. 

Algorithm~\ref{Alg: W1-FE} (called {W1-FE}) is designed to simulate the discretization of the gradient-flow ODE. It first computes an estimate of the Kantorovich potential by following the discriminator training in a typical WGAN algorithm. With the Kantorovich potential estimated, a generator is trained to move along the ODE discretization. As the generator's task is to learn the new distribution at the next time point, {\it persistent training} techniques \citep{Fischetti18} can be naturally incorporated---namely, given a set of samples from the new distribution, the generator descends the loss $K\in\N$ consecutive times to better represent these samples (and thus the new distribution). For the case $K=1$ (which means no persistent training), Proposition~\ref{Prop: W1-GANs simulate our process} shows that W1-FE reduces to WGAN. Despite the coincidence under $K=1$, W1-FE and WGAN are fundamentally different. Specifically, if we also incorporate persistent training into WGAN, the generator updates in W1-FE and WGAN can be quite different under any $K>1$. That is, $K=1$ is the only case where they agree; see the discussion below Proposition~\ref{Prop: W1-GANs simulate our process} and Remark~\ref{rem:W1-FE and WGAN} for details. In fact, adding persistent training to WGAN may destabilize the training process and worsen the training results. This is in contrast to what we observe in W1-FE, to be introduced below. 

We train W1-FE with diverse persistency levels $K=1,3,5,10$ in three experiments that involve datasets from low to high dimensions, including synthetic two-dimensional mixtures of Gaussians and real datasets of USPS, MNIST, and CIFAR-10. Across the three experiments, W1-FE with $K>1$ converges significantly faster and achieves better training results than the baseline $K=1$ case (which is WGAN). Intriguingly, while a larger $K$ generally implies faster convergence, we observe numerically a threshold of $K$ beyond which the training results start to deteriorate, possibly due to overfitting. This suggests that taking $K$ to be at that threshold can likely best balance the benefits of persistent training against potential overfitting; see the last paragraph in Section~\ref{sec:limitations} for details.  

Many recently developed generative models also feature ``Wasserstein gradient flows'' (e.g., \citet{Fan22}, \citet{Choi24}, \citet{Zhang23}, \citet{Onken21}). The common theme in these works is to minimize a loss function under the geometry induced by the Wasserstein-$2$ distance. This means that ``Wasserstein gradient flows'' in the literature should be more precisely referred to as ``Wasserstein-2 gradient flows.'' Since subdifferential calculus under the Wasserstein-$2$ distance and the resulting gradient flows are widely studied \citep[Chapter 10 and Section 11.2]{Ambrosio08}, much of the recent literature immediately leverages these results to design new generative models. Our study is distinct from all this, as we immediately tackle Wasserstein-1 gradient flows, which are far less understood. As explained before, this focus on Wasserstein-1 gradient flows allows us to recover and even improve WGAN, which by construction minimizes the Wasserstein-1 distance and cannot be easily analyzed by the standard Wasserstein-2 framework.

The rest of this paper is organized as follows. Section~\ref{sec:preliminaries} introduces the mathematical framework and notation to be used. Section~\ref{sec:PF} discusses how the ``gradient descent'' idea can be applied to minimizing the Wasserstein-1 distance and formulates the corresponding gradient-flow ODE. Section~\ref{sec:discretization} defines a forward Euler discretization of the ODE and designs an algorithm (i.e., W1-FE) to simulate this discretization. Section~\ref{sec:examples} trains W1-FE in three experiments and demonstrates its superiority over WGAN. Section~\ref{sec:limitations} discusses the limitations of our study on both the theoretical and numerical sides. Section~\ref{sec:conclusion} concludes our findings. 

%%%%%%%%%%%%%%%%%%%%%%%%%%%%
%%%%%%%%%%%%%%%%%%%%%%%%%%%%

\section{Mathematical Preliminaries}\label{sec:preliminaries}
Fix $d\in \N$ and let $\mathcal L^d$ be the Lebesgue measure on $\R^d$. Let $\mathcal P(\R^d)$ be the set of probability measures on $\R^d$ and $\mathcal{P}_{p}(\R^d)$, for $p\ge 1$, be the set of elements in $\mathcal{P}(\R^d)$ with finite $p^{th}$ moments, i.e., 
\[
 \mathcal{P}_{p}(\R^d) := \bigg\{ \mu \in \mathcal{P}(\R^d) : \int_{\R^d} |y|^p d\mu(y) < \infty \bigg\}.
\]
The Wasserstein-$p$ distance, a metric on $\Pc_p(\R^d)$, is defined by
%For any $\mu,\nu\in \mathcal{P}_{2}(\mathbb{R}^{d})$,
%The $2$ Wasserstein distance, $W_{2}: \mathcal{P}_{2}(\mathbb{R}^{d}) \times \mathcal{P}_{2}\mathbb{R}^{d} \rightarrow [0, \infty)$ is defined in \cite[7.1.1]{Ambrosio08} as 
\begin{equation*} \label{Eq: Kantorovich Relaxation}
    \begin{aligned}
        W_{p}(\mu, \nu) := &\left( \inf_{\gamma \in \Gamma(\mu,\nu)} \int_{\R^d \times \R^d} |x - y|^p \, d\gamma(x, y)\right)^{1/p}
        \\
        &\quad \text{for all} \; \mu,\nu\in \mathcal{P}_{p}(\R^d),
    \end{aligned}
\end{equation*}
where $\Gamma(\mu,\nu)$ is the set of all probability measures on $\R^d\times\R^d$ whose marginals on the first and second coordinates are $\mu$ and $\nu$, respectively \citep[Definition 6.1]{Villani09}.  For $p=1$, we recall the Kantorovich-Rubinstein duality formula for the $W_{1}$ distance \citep[(5.11)]{Villani09}, i.e., 
\begin{eqnarray} \label{Eq: W1 Duality Formula}
    W_{1}(\mu, \nu) = \sup_{\substack{\|\varphi\|_{\operatorname{Lip}} \leq 1}} \left \{ \int_{\R^d} \varphi \, d(\mu - \nu) \right \},
\end{eqnarray}
where ``$||\varphi||_{\Lip}\le 1$'' denotes the set of all $\varphi:\R^d\to\R$ that are 1-Lipschitz functions. % (i.e., its Lipschitz constant is bounded by 1).

\begin{definition}\label{def:KP}
%Given $\mu,\nu\in\Pc_1(\Xc)$, a minimizing $\gamma\in\Gamma(\mu,\nu)$ for \eqref{Eq: Kantorovich Relaxation} is called an optimal transport plan from $\mu$ to $\nu$ and we denote by $\Gamma_0(\mu,\nu)$ the collection of all such plans. In addition, 
A $1$-Lipschitz $\varphi:\R^d\to\R$ that maximizes \eqref{Eq: W1 Duality Formula} is called a Kantorovich potential from $\mu$ to $\nu$ and will be denoted by $\varphi_\mu^\nu$. % to emphasize its dependence on $\mu$ and $\nu$. 
\end{definition}
The general definition of a (maximal) Kantorovich potential is well-known for any $p\ge1$; see the remark above \citet[Theorem 6.15]{Ambrosio08}. For $p=1$, it reduces specifically to Definition~\ref{def:KP}, thanks to the discussion in \citet[Particular Case 5.4]{Villani09}. 

\begin{remark}\label{rem:nabla phi}
For any $\mu,\nu\in\Pc_1(\R^d)$, by \citet[Theorem 5.10 (iii)]{Villani09}, a Kantorivich potential $\varphi_\mu^\nu$ generally exists. As $\varphi_\mu^\nu$ is $1$-Lipschitz, $\nabla \varphi_\mu^\nu(x)$ exists with $|\nabla \varphi_\mu^\nu(x)|\le 1$ for $\mathcal L^d$-a.e.\  $x\in\R^d$.   
 \end{remark}

%%%%%%%%%%%%%%%%%%%
%%%%%%%%%%%%%%%%%%%

\section{Problem Formulation}\label{sec:PF}
Let $\mud\in\Pc_1(\R^d)$ denote the (unknown) data distribution. Starting with an arbitrary initial estimate $\mu_0\in\Pc_1(\R^d)$ of  $\mud$, we aim to improve our estimate progressively and ultimately solve the problem
\begin{equation}\label{to solve}
\min_{\mu\in\Pc_1(\R^d)} W_1(\mu,\mud). 
\end{equation}
As it can be checked directly that $\mu\mapsto W_1(\mu,\mud)$ is convex on $\Pc_1(\R^d)$ (Appendix~\ref{subsec:convexity of J}), it is natural to ask if \eqref{to solve} can be solved by gradient descent, the traditional wisdom of convex minimization. Recall that for a convex $f :\R^d\to \R$, if its minimizer $y^*\in\R^d$ exists, it can be found by gradient descent in $\R^d$. Specifically, for any initial point $y\in\R^d$, the ODE 
\begin{equation}\label{classical ODE}
d Y_t = -\nabla f(Y_t) dt,\quad Y_0=y\in\R^d
\end{equation}
converges to $y^*$ as $t\to\infty$. For the derivation of a similar gradient-descent ODE for \eqref{to solve}, where the minimizer is clearly $\mud\in\Pc_1(\R^d)$, the crucial question is how the ``gradient'' of the function
\begin{equation}\label{J}
J(\mu) := W_{1}(\mu, \mud),\quad \mu\in\Pc_1(\R^d)
\end{equation}
should be defined. As mentioned in the introduction, subdifferential calculus is well-developed in $\Pc_p(\R^d)$ for all $p>1$ \citep[Chapter 10]{Ambrosio08}, but the same construction breaks down exactly when $p=1$. As a result, neither ``Fr\'{e}chet differential'' \citep[Corollary 10.2.7]{Ambrosio08} nor the equivalent ``Wasserstein gradient'' \citep[Definition 5.62]{CD-book-18-I} are well-defined in $\Pc_1(\R^d)$. %, in contrast to \citet{Huang24} which focuses on $\Pc_2(\Xc)$. 
To circumvent this, let us first recall %we notice that a general gradient notion can potentially be defined using a 
``linear functional derivative'' from the mean field game literature. 

\begin{definition}\label{def:LFD}
Let $\mathcal S\subseteq \Pc(\R^d)$ be convex. A linear functional derivative of $U: \mathcal S \rightarrow \mathbb{R}$ is a function $\frac{\delta U}{\delta m}:\mathcal S\times\R^d\to\R^d$ that satisfies
%admits  at $\mu \in \Pc_{1}(\Xc)$ if there exists a real valued measurable function $\mathbb{R}^{d} \ni y \rightarrow \frac{\delta U}{\delta m}(\mu,y)$ such that
    \begin{equation} \label{Cond: LFD}
        \begin{aligned}
            \lim_{\epsilon \rightarrow 0^{+}} &\frac{U(\mu + \epsilon(\nu - \mu)) - U(\mu)}{\epsilon} = \\
            &\int_{\R^d} \frac{\delta U}{\delta m}(\mu,y) \, d(\nu - \mu)(y),\quad \forall \mu, \nu \in \mathcal S. 
        \end{aligned}
    \end{equation}
\end{definition}
The above definition is in line with \citet[Definition 2.1]{Jourdain21}, where $\mathcal S = \Pc_p(\R^d)$ for $p\ge 1$, and \citet[Definition 5.43]{CD-book-18-I}, where $\mathcal S=\Pc_2(\R^d)$. Note that ``${\delta U}/{\delta m}$'' is simply a common notation for a function satisfying \eqref{Cond: LFD}, where ``$m$'' indicates that the variable in discussion is a probability measure and ``${\delta}/{\delta m}$'' alludes to a kind of differentiation with respect to $m$. 

The key observation here is that the ``Euclidean gradient of a linear functional derivative,'' i.e., $\nabla\frac{\delta U}{\delta m}(\mu,\cdot):\R^d\to\R^d$, can generally serve as the ``gradient of $U$ at a measure $\mu$.'' For $\mathcal S=\Pc_2(\R^d)$, \citet[Proposition 5.48 and Theorem 5.64]{CD-book-18-I} show that $\nabla\frac{\delta U}{\delta m}(\mu,\cdot)$ in fact coincides with the Wasserstein gradient of $U$ at $\mu\in\Pc_2(\R^d)$. %, well-defined now in $\Pc_2(\Xc)$. 
For $\mathcal S=\Pc^r(\R^d) :=\{\mu\in\Pc(\R^d):\mu\ll \mathcal L^d, \frac{d\mu}{d\mathcal L^d}\in C^1(\R^d)\}$, where Wasserstein gradients are not well-defined, \citet{Huang23} show that $\nabla \frac{\delta U}{\delta m}(\mu, \cdot)$ still fulfills a gradient-type property. Specifically, for any $\mu\in\Pc^r(\R^d)$ and $\xi:\R^d\to\R^d$, let $\mu^\xi_\epsilon$ be the law of $Y+\epsilon \xi(Y)$, where $Y$ is a random variable whose law is $\mu$. For sufficiently smooth and compactly supported $\xi$,  \citet[Proposition 5]{Huang23} shows that $\mu^\xi_\epsilon\in\Pc^r(\R^d)$ and
\[
\lim_{\epsilon\to 0^+} \frac{U(\mu^\xi_\epsilon)-U(\mu)}{\epsilon} = \int_{\R^d} \nabla \frac{\delta U}{\delta m}(\mu,y) \cdot \xi(y) d\mu(y),
\]
provided that $\frac{\delta U}{\delta m}$ is locally integrable and sufficiently continuous. That is, for any $y\in\R^d$, $\nabla \frac{\delta U}{\delta m}(\mu,y)$ specifies how moving along $\xi(y)$ instantaneously changes the function value from $U(\mu)$, which suggests that $\nabla \frac{\delta U}{\delta m}(\mu,\cdot):\R^d\to\R^d$ should be the proper ``gradient of $U$ at $\mu\in\Pc^r(\R^d)$.''

In view of this, in our case of $\mathcal S=\Pc_1(\R^d)$, where Wasserstein gradients are again not well-defined, we take the ``gradient of $J$ in \eqref{J} at $\mu\in\Pc_1(\R^d)$'' to be $\nabla\frac{\delta J}{\delta m}(\mu, \cdot):\R^d\to\R^d$. The resulting gradient-descent ODE for \eqref{to solve}, in analogy to \eqref{classical ODE}, is then 
%Recently, \citet[Section 2]{Huang23} suggested that for a function $U:\Pc(\Xc)\to \R$, where $\Pc(\Xc)$ denotes the space of all probability measures, we should take its gradient to be the ``Euclidean gradient of its linear functional derivative.'' This is because at each $\mu\in\Pc(\Xc)$, $y\mapsto \nabla \frac{\delta U}{\delta m}(\mu, y)$ satisfies a gradient-type property; see \citet[Proposition 5]{Huang23}. This means that the gradient-descent ODE for \eqref{to solve} can be stated as 
\begin{equation}\label{first Y}
    dY_{t} = -\nabla \frac{\delta J}{\delta m}(\mu^{Y_{t}}, Y_{t})\, dt, \quad \mu^{Y_{0}} = \mu_{0} \in \Pc_{1}(\R^d).
\end{equation}
This ODE, intriguingly, is {\it distribution-dependent}. At time 0, $Y_0$ is an $\R^d$-valued random variable whose law is $\mu_0\in\Pc_1(\R^d)$, an arbitrarily specified initial distribution. This initial randomness trickles through the ODE dynamics in \eqref{first Y}, such that $Y_t$ remains an $\R^d$-valued random variable, with its law denoted by $\mu^{Y_t}\in\Pc_1(\R^d)$, at every $t>0$. The evolution of the ODE is then determined jointly by the ``gradient of $J$'' at the present distribution $\mu^{Y_t}\in\Pc_1(\R^d)$ (i.e., %Euclidean gradient of $J$'s linear functional derivative (i.e., 
the function $\nabla \frac{\delta J}{\delta m}(\mu^{Y_{t}}, \cdot)$) and the actual realization of $Y_t$ (which is plugged into $\nabla \frac{\delta J}{\delta m}(\mu^{Y_{t}}, \cdot)$).

To ensure that ODE \eqref{first Y} is tractable enough, one needs to show that $\frac{\delta J}{\delta m}$ exists and admits a concrete characterization. Our first main theoretic result serves this purpose. 
%we compute $\frac{\delta J}{\delta m}$.

\begin{proposition}\label{Prop: LFD of W1 distance}
For any $\mu\in\Pc_1(\R^d)$, a Kantorovich potential $ \varphi_{\mu}^{\mud}$ (Definition~\ref{def:KP}) is a linear functional derivative of $J : \Pc_{1}(\R^d) \rightarrow \mathbb{R}$ in \eqref{J} at $\mu \in \Pc_{1}(\R^d)$ (Definition~\ref{def:LFD} with $\mathcal S=\Pc_1(\R^d)$). Specifically, for any $\mu\in\Pc_1(\R^d)$, 
    \begin{equation*}
        \frac{\delta J}{\delta m}(\mu, y) = \varphi_{\mu}^{\mud}(y)\quad \forall y\in\R^d. 
    \end{equation*}
\end{proposition}

The proof of Proposition~\ref{Prop: LFD of W1 distance} is relegated to Appendix~\ref{subsec:proof of Prop: LFD of W1 distance}. 
%\begin{remark}
To the best of our knowledge, Proposition~\ref{Prop: LFD of W1 distance} is the first result that establishes a precise connection between ``Kantorovich potential'' in optimal transport and ``linear functional derivative'' in the mean field game literature. 
%\end{remark}

Thanks to Proposition~\ref{Prop: LFD of W1 distance}, ODE \eqref{first Y} now becomes 
\begin{equation}\label{Eq: ODE}
    dY_{t} = -\nabla \varphi_{\mu^{Y_{t}}}^{\mud}(Y_{t}) \, dt, \quad \mu^{Y_{0}} = \mu_{0} \in \Pc_{1}(\R^d).
\end{equation}
That is, the evolution of the ODE is determined jointly by a Kantorovich potential from the present distribution $\mu^{Y_t}$ to $\mud$ (i.e., the function $\varphi_{\mu^{Y_t}}^{\mud}(\cdot)$) and the actual realization of $Y_t$ (which is plugged into $\nabla\varphi_{\mu^{Y_t}}^{\mud}(\cdot)$).

%\begin{remark}\label{Rem: Gradient Flow and OT}
%In light of Remark~\ref{Rem: Grad potential parallel to TR}, we may view \eqref{Eq: ODE} as transporting mass along the transport rays for $\mu^{Y_{t}}$ and $\mu_{d}$. In other words, the ``negative gradient of $J$" (i.e., $-\nabla \frac{\delta J}{\delta m}(\mu^{Y_t},\cdot)$ in \eqref{first Y}) directs mass from $\mu^{Y_{t}}$ towards that of $\mu_{d}$ along the most efficient path possible. This suggests a deep connection between our gradient flow idea and the Wasserstein$-1$ theory of optimal transportation.
%\end{remark}

%%%%%%%%%%%%%%%%%%%%%%%%%%%%
%%%%%%%%%%%%%%%%%%%%%%%%%%%%

\section{A Discretization of ODE \eqref{Eq: ODE}}\label{sec:discretization}
Given $\epsilon > 0$ and an initial random variable $Y_{0,\epsilon}=Y_0$ with a given law $\mu^{Y_{0}} = \mu_{0}\in\Pc_1(\R^d)$, we consider a new random variable defined by
\begin{equation*}
    Y_{1,\epsilon} := Y_{0,\epsilon} - \epsilon \nabla \varphi_{\mu^{Y_{0,\epsilon}}}^{\mud}(Y_{0,\epsilon}).
\end{equation*}
Note that this is the very first step, from time 0 to time $\epsilon$, in a forward Euler discretization of ODE \eqref{Eq: ODE}. Using the law of $Y_{1,\epsilon}$, denoted by $\mu^{Y_{1,\epsilon}}$, we can obtain a Kantorovich potential $\varphi_{\mu^{Y_{1,\epsilon}}}^{\mud}$ from the present distribution $\mu^{Y_{1,\epsilon}}$ at time $\epsilon$ to $\mud$. This allows us to perform another forward Euler update and get
%\begin{equation*}
   $ Y_{2,\epsilon} := Y_{1, \epsilon} - \epsilon \nabla \varphi_{\mu^{Y_{1,\epsilon}}}^{\mud}(Y_{1, \epsilon})$.
%\end{equation*}
We may continue this procedure and obtain a sequence of random variables $\{Y_{n, \epsilon}\}_{n\in\N}$, with
\begin{equation} \label{Discrete Time W1 Transport Process}
    Y_{n, \epsilon} := Y_{n-1, \epsilon} - \epsilon \nabla \varphi_{\mu^{Y_{n-1,\epsilon}}}^{\mud}(Y_{n-1, \epsilon}),\quad\forall n\in\N. 
\end{equation}
This discretization recursively defines a sequence of measures $\{ \mu^{Y_{n-1,\epsilon}} \}_{n\in\N}$ in $\Pc_{1}(\R^d)$. %, for $\epsilon \in (0, \epsilon_{0})$ for some $\epsilon_{0} > 0$ small enough. 
%We should also note that regardless of the well-posedness of \eqref{Eq: ODE}, the process $\{ Y_{n,\epsilon}\}$ generated by \eqref{Discrete Time W1 Transport Process} is defined $\mathcal{L}^{d}$ a.e. For the rest of this analysis, We assume that each measure $\mu^{Y_{n,\epsilon}}$ is dominated by the Lebesgue measure (ie, admits a density function). 
A piecewise constant flow of measures $\mu_{\epsilon}: [0, \infty)\rightarrow \mathcal{P}_{1}(\R^d)$ can then be defined by
\begin{equation} \label{Def: Interpolating measure curve}
    \mu_{\epsilon}(t) := \mu^{Y_{n-1, \epsilon}} \quad \hbox{for}\ \ t \in [(n-1)\epsilon, n\epsilon),\ \  n\in\N.
\end{equation}
Our second main theoretic result establishes the convergence of $\mu_{\epsilon}$ as $\epsilon\to 0^+$. Its proof is relegated to Appendix~\ref{subsec:proof of convergence result}.

%Now we state our main convergence result. A complete proof is provided in the appendix.

\begin{theorem}\label{Th: Convergence of interpolating measure curve}
 For any $\epsilon>0$, let $\mu_{\epsilon}:[0,\infty) \rightarrow \mathcal{P}_{1}(\mathbb{R}^{d})$ be defined as in \eqref{Def: Interpolating measure curve} and assume that $\mu_{\epsilon}(t) \ll \mathcal{L}^{d}$ for all $t \ge 0$. % and $\epsilon \in (0,\epsilon_{0})$, we have $\mu_{\epsilon}(t) << \mathcal{L}^{d}$. 
 Then, there exists a sequence $\{ \epsilon_{k} \}_{k\in\N}$ with $\epsilon_k\to 0^+$ 
 %$\subset (0, \epsilon_{0})$ 
 and a curve $\mu^{*}: [0,\infty) \rightarrow \mathcal{P}_{1}(\mathbb{R}^{d})$ such that
    \begin{equation*}
        \lim_{k \rightarrow \infty} W_{1}(\mu_{\epsilon_{k}}(t),\mu^{*}(t)) = 0\quad \forall t>0. 
    \end{equation*}
   Furthermore, $t\mapsto \mu^{*}(t)$ is uniformly continuous (in the $W_{1}$ sense) on compacts of $[0,\infty)$.
\end{theorem}

%\begin{remark}
At first glance, one might suspect that Theorem~\ref{Th: Convergence of interpolating measure curve} is a straightforward extension of \citet[Theorem 5.2]{Huang24} from $\Pc_2(\R^d)$ to the larger space $\Pc_1(\R^d)$. In fact, proving Theorem~\ref{Th: Convergence of interpolating measure curve} requires completely different arguments. \citet[Theorem 5.2]{Huang24} is established by an interpolation argument: for any $p>1$, the most cost-efficient path to move one probability measure to another in $\Pc_p(\R^d)$, measured by the $W_p$ distance, can be interpolated recursively into smaller segments between intermediate measures, where each segment is most cost-efficient in itself \citep[Lemma 7.2.1]{Ambrosio08}. This allows a suitable ODE discretization to correspond to the smaller segments \citep[Proposition 5.3]{Huang24}, thereby admitting a well-defined limit (i.e., the whole most cost-efficient path). The interpolation result in \citet[Lemma 7.2.1]{Ambrosio08}, however, does not hold for $p=1$. Theorem~\ref{Th: Convergence of interpolating measure curve} instead relies on the uniform boundedness of the gradient of {\it any} Kantorovich potential (Remark~\ref{rem:nabla phi}), which is unique to the $p=1$ case. Such boundedness allows us to show that $\{\mu_{\epsilon}(t): \epsilon>0, t\in[0,T]\}$ in \eqref{Def: Interpolating measure curve} is compact in $\Pc_1(\R^d)$ for any $T>0$ and the curve $t\mapsto \mu_{\epsilon}(t)$ becomes equicontinuous as $\epsilon\to 0^+$, so that a refined Arzela-Ascoli argument can be applied to give the convergence of $\mu_\epsilon$; see Appendix~\ref{subsec:proof of convergence result}. 
% for details. 
%\end{remark}

\begin{remark}\label{rem:robustness}
Theorem~\ref{Th: Convergence of interpolating measure curve} is important for the numerical implementation: it asserts that our discretization scheme \eqref{Discrete Time W1 Transport Process} is stable for small time steps and there is a well-defined limit. In light of its proof in Appendix~\ref{subsec:proof of convergence result}, the established convergence actually holds in a more general scenario. If the discretization \eqref{Discrete Time W1 Transport Process} is modified to 
\begin{equation*} %\label{Discrete Time W1 Transport Process}
    Y_{n, \epsilon} := Y_{n-1, \epsilon} - \epsilon \nabla g_{n-1,\epsilon}(Y_{n-1, \epsilon}),\quad\forall n\in\N,
\end{equation*}
where each $g_{n-1,\epsilon}$ is a 1-Lipschitz function (but not necessarily the Kantorovich potential $\varphi_{\mu^{Y_{n-1,\epsilon}}}^{\mud}$), the same convergence result in Theorem~\ref{Th: Convergence of interpolating measure curve} remains true. This is because the arguments in Appendix~\ref{subsec:proof of convergence result} hinge on only the 1-Lipschitz continuity of $g_{n-1,\epsilon}$ (for $|\nabla g_{n-1,\epsilon}|$ to be bounded by 1), but not the specific form of $g_{n-1,\epsilon}$. That is, our discretization scheme \eqref{Discrete Time W1 Transport Process} is robust in the following sense: in actual computation, as long as the estimated $\varphi_{\mu^{Y_{n-1,\epsilon}}}^{\mud}$ is 1-Lipschitz (facilitated by the discriminator's regularization in \citet{Gulrajani17} and \citet{Petzka18}), the scheme remains stable for small time steps with a well-defined limit. 
\end{remark}

Thanks to the convergence result in Theorem~\ref{Th: Convergence of interpolating measure curve}, 
%With this convergence result, we know that the Euler scheme \eqref{Discrete Time W1 Transport Process} is well posed and that a numerical implementation of such a scheme is stable for small time step. 
we propose an algorithm (called \textbf{W1-FE}) to simulate \eqref{Discrete Time W1 Transport Process}; see Algorithm~\ref{Alg: W1-FE}. We use two neural networks to carry out the simulation, one for the Kantorovich potential $\varphi:\R^d \rightarrow \mathbb{R}$ and the other for the generator $G_{\theta} : \R^\ell \rightarrow \R^d$, where $\mathbb{R}^{\ell}$ (with $\ell\in \mathbb{N}$) is the latent space where we sample priors. To compute $\varphi$, we can follow the discriminator training in any well-known WGAN algorithm, e.g., vanilla WGAN from \citet{Arjovsky17}, W1-GP from \citet{Gulrajani17}, or W1-LP from \citet{Petzka18}, to obtain an estimate of the Kantorovich potential from the distribution of samples generated by $G_{\theta}$ to that of the data, $\mud$. % (which is the discriminator of these algorithms). 
To allow such generality in Algorithm~\ref{Alg: W1-FE}, we simply denote the computation of $\varphi$ by \texttt{SimulatePhi}$(\theta)$ and treat it as a black box function. When we use the method of \citet{Petzka18} (or \citet{Gulrajani17}) to compute $\varphi$, Algorithm~\ref{Alg: W1-FE} will be referred to as \textbf{W1-FE-LP} (or \textbf{W1-FE-GP}). 

The generator $G_{\theta}$ is trained by explicitly following the (discretized) ODE \eqref{Discrete Time W1 Transport Process}. We start with a set of priors $\{ z_{i} \}$, produce a sample $y_{i}=G_\theta(z_i)$ of $\mu^{Y_{n, \epsilon}}$, and  use a forward Euler step to compute a sample $\zeta_{i}$ of $\mu^{Y_{n+1,\epsilon}}$. The generator's task is to learn how to produce samples indistinguishable from the points $\{ \zeta_{i}\}$---or more precisely, to learn the distribution $\mu^{Y_{n+1,\epsilon}}$, represented by the points $\{ \zeta_{i}\}$. To this end, we fix the points $\{\zeta_i\}$ and update the generator $G_\theta$ by descending the mean square error (MSE) between $\{G_\theta(z_i)\}$ and $\{\zeta_i\}$ up to $K\in\N$ times. Note that throughout the $K$ updates of $G_\theta$, the points $\{\zeta_i\}$ are kept unchanged. This sets us apart from the standard implementation of stochastic gradient descent (SGD), but for a good reason: as our goal is to learn the distribution represented by $\{\zeta_i\}$, it is important to keep $\{\zeta_i\}$ unchanged for the eventual $G_{\theta}$ to more accurately represent $\mu^{Y_{n+1,\epsilon}}$, such that the (discretized) ODE \eqref{Discrete Time W1 Transport Process} is more accurately simulated. 

%With the new sample, we may descend the mean square error (MSE) up to some $K$ times. The generator update step amounts to a simple supervised learning problem, for there is a one to one correspondence between $\{y_{i}\}$ and $\{\zeta_{i}\}$. In particular, each particle $y_{i}$ evolves by an Euler step to produce a new particle $\zeta_{i}$. 

How we update the generator $G_\theta$ corresponds to \textit{persistent training} in \citet{Fischetti18}, a technique that reuses the same minibatch for $K$ consecutive SGD iterations. Experiments in \citet{Fischetti18} show that using a \textit{persistency level} of five (i.e., taking $K=5$) achieves much faster convergence on the CIFAR-10 dataset \citep[Figure 1]{Fischetti18}. In our numerical examples (see Section~\ref{sec:examples}), we will also show that increasing the persistency level appropriately in W1-FE can markedly improve training performance. 

%The generator's task is to then learn how to produce a sample indistinguishible from the corresponding $\{ \zeta_{i} \}$. We may descend the mean square error (MSE) between $\{\zeta_{i}\}$ and generated samples up to some $K$ times. In other words, we train the generator with persistency level $K$. By utilizing persistent training, we ensure $G_{\theta}$ may more accurately generate a sample of $\mu_{n+1,\epsilon}$, that is, the model follows \eqref{Discrete Time W1 Transport Process} more efficiently. Persistent training is not common in generative modeling, but the idea does have precedence in machine learning (see \cite{Fischetti18}). 

\begin{algorithm}[tb]
   \caption{W1-FE, our proposed algorithm.}
   \label{Alg: W1-FE}
\begin{algorithmic}
   \STATE {\bfseries Input:} measures %$\mu_{0}$ 
   $\mu^{\bm{z}}$, $\mud$, batch sizes $m\in\N$, generator learning rate $\gamma_{g}>0$, time step $\epsilon>0$, persistency value $K\in\N$, function \texttt{SimulatePhi} to approximate Kantorovich potential, generator $G_{\theta}$ parameterized as a deep neural network.
   \FOR{Number of training epochs}
   \STATE $\varphi \gets \texttt{SimulatePhi}(\theta)$ 
   \STATE Sample a batch $(z_{1}, \cdots, z_{m})$ from $\mu^{\bm{z}}$
   \STATE Compute $y_{i} \gets G_{\theta}(z_{i})$
   \STATE Compute $\zeta_{i} \gets y_{i} - \epsilon \nabla \varphi(y_{i})$. 
        \FOR{$K$ generator updates}
            \STATE Update $\theta \gets \theta - \frac{\gamma_{g}}{m}\nabla_{\theta} \sum_{i} |\zeta_{i}- G_{\theta}(z_{i})|^{2}$.
        \ENDFOR
   \ENDFOR
\end{algorithmic}
\end{algorithm}

\subsection{A Comparison: W1-FE and WGAN} \label{Sec: W1-FE vs WGAN}
Our first finding is that W1-FE actually covers WGAN as a special case. 
%A close comparison between W1-FE and WGAN yields two important findings. First, while W1-FE is developed by a gradient-flow approach, distinct from the min-max game perspective of WGAN, W1-FE actually covers WGAN as a special case. 
For the case $K=1$ in Algorithm~\ref{Alg: W1-FE} (i.e., W1-FE), the generator update reduces to standard SGD without persistent training, which turns Algorithm~\ref{Alg: W1-FE} into a standard WGAN algorithm.

\begin{proposition} \label{Prop: W1-GANs simulate our process}
    The WGAN algorithms presented in \citet{Arjovsky17}, \citet{Gulrajani17}, \citet{Petzka18} are special cases of Algorithm~\ref{Alg: W1-FE} with $K=1$.
\end{proposition}

\begin{proof}
Take \texttt{SimulatePhi} in Algorithm~\ref{Alg: W1-FE} to be the discriminator update in an aforementioned WGAN algorithm, so that the produced $\varphi$ is exactly the estimated Kantorovich potential in the WGAN algorithm. 
%By setting \texttt{SimulatePhi} to be how an aforementioned WGAN algorithm approximates Kantorovich potentials, we see that $\varphi$ in Algorithm~\ref{Alg: W1-FE} is exactly the discriminator in the corresponding WGAN algorithm. 
Then, it suffices to show that the generator update in Algorithm~\ref{Alg: W1-FE}, with $K=1$, coincides with that in the WGAN algorithm. When $K = 1$,
\begin{align}\label{K=1 equivalence}
\begin{split}
\nabla_\theta \frac{1}{m} \sum_{i=1}^{m} |\zeta_{i} - &G_{\theta}(z_{i})|^{2}\\ &= -\frac{2}{m} \sum_{i = 1}^{m} (\zeta_{i} - G_{\theta}(z_{i})) \nabla_{\theta}G_{\theta}(z_{i})\\
&= \frac{2}{m} \sum_{i = 1}^{m} \epsilon \nabla \varphi(G_{\theta}(z_{i})) \nabla_{\theta}G_{\theta}(z_{i})\\
&=\frac{2 \epsilon}{m} \nabla_{\theta} \sum_{i=1}^{m} \varphi (G_{\theta}(z_{i})),
\end{split}
\end{align}
where the second equality stems from $\zeta_{i} = G_{\theta}(z_{i}) - \epsilon \nabla \varphi(G_{\theta}(z_{i}))$, due to the two lines above the generator update in Algorithm~\ref{Alg: W1-FE}. That is, the generator update in Algorithm~\ref{Alg: W1-FE} is $\theta \gets \theta - \gamma_{g} \frac{2 \epsilon}{m} \nabla_{\theta} \sum_{i=1}^{m} \varphi(G_{\theta}(z_{i}))$, the same as that in the WGAN algorithm with a learning rate $2\gamma_{g}\epsilon$.
\end{proof}

The above result is somewhat unexpected: after all, W1-FE builds upon our gradient-flow approach, distinct from the min-max game perspective that underlies WGAN. Proposition~\ref{Prop: W1-GANs simulate our process} shows that the two fundamentally different methods may coincide when the computation of the gradient flow, or the (discretized) ODE \eqref{Discrete Time W1 Transport Process}, is {\it crude}---in the sense that $\mu^{Y_{n+1,\epsilon}}$, the distribution along the ODE at the next time step, is less accurately approximated (under $K=1$). 

\begin{remark}\label{rem:W1-FE and WGAN}
From the proof of Proposition~\ref{Prop: W1-GANs simulate our process}, it is tempting to think that if one enforces persistent training also on the generator update in WGAN (i.e., performs the update
\begin{equation} \label{Eq: WGAN Update}
    \theta \leftarrow \theta - \frac{\gamma_{g}}{m} \nabla_{\theta} \sum_{i=1}^{m} \varphi(G_{\theta}(z_{i}))
\end{equation}
$K\in\N$ consecutive times with the same minibatch $\{ z_{i} \}$), %, without updating the discriminator $\varphi$. 
WGAN will become our W1-FE. This is however not the case. When $G_\theta$ is updated for the second time in the last line of W1-FE, the second equality in \eqref{K=1 equivalence} no longer holds, as ``$\zeta_{i} = G_{\theta}(z_{i}) - \epsilon \nabla \varphi(G_{\theta}(z_{i}))$'' is true only when $G_\theta$ is obtained from the previous iteration and has not been further updated yet. Hence, starting from $K=2$, the connection between W1-FE and WGAN breaks down. That is, even when persistent training is included in WGAN, the generator updates in W1-FE and WGAN coincide only for $K = 1$, and may be quite different for $K > 1$. 
%A casual observer may consider whether one can utilize persistent training for any of the previously discussed WGAN algorithms. However, in those algorithms, the update rule explicitly uses the potential function $\varphi$. As such, once one updates the generator $G_{\theta}$, it is not necessarily true that $\varphi$ is still the corresponding Kantorovich potential for the updated $G_{\theta}$. Therefore, after one generator update, its loss is no longer $\frac{1}{m} \sum_{i=1}^{m} \varphi(G_{\theta}(z_{i}))$. As we shall see in the next section, this insight has significant consequences for training.
\end{remark}

Our second finding, as a follow-up to Remark~\ref{rem:W1-FE and WGAN}, is that persistent training actually hurts the performance of WGAN, whereas it can significantly enhance the performance of W1-FE (as mentioned in the paragraph above Section~\ref{Sec: W1-FE vs WGAN} and demonstrated in Section~\ref{sec:examples} below). To see this, recall the min-max formulation of WGAN, i.e., 
\begin{align*} %\label{Eq: True Problem}
    &\min_{\theta} W_{1}\big((G_{\theta})_{\#}\mu^{\bm{z}}, \mud\big)\notag\\
     =\  &\min_{\theta} \max_{||\varphi||_{\Lip}\leq1} \left\{ \mathbb{E}_{\bm{z}\sim \mu^{\bm{z}}} [\varphi(G_{\theta}(\bm{z}))] - \mathbb{E}_{\bm{x}\sim\mud}[\varphi(\bm{x})]  \right\},
\end{align*}
where $(G_{\theta})_{\#}\mu^{\bm{z}}$ is the probability measure on $\R^d$ induced by $G_\theta(\bm{z})$ with $\bm{z}\sim\mu^{\bm{z}}$ and the equality follows from the duality \eqref{Eq: W1 Duality Formula}; see \citet[Section 2.2]{Gulrajani17} for an equivalent min-max setup. % distributed by a prior distribution $\mu^Z$. 
 For any fixed $\theta$, by Definition~\ref{def:KP} and Remark~\ref{rem:nabla phi}, the inside maximization over $||\varphi||_{\Lip}\leq1$ admits a maximizer $\varphi_\theta$ (i.e., an optimal discriminator in response to the generator $G_\theta$). We therefore obtain
\begin{align} \label{min phi_theta}
    &\min_{\theta} W_{1}\big((G_{\theta})_{\#}\mu^{\bm{z}}, \mud\big)\notag\\
     =\  &\min_{\theta}  \left\{ \mathbb{E}_{\bm{z}\sim \mu^{\bm{z}}} [\varphi_\theta(G_{\theta}(\bm{z}))] - \mathbb{E}_{\bm{x}\sim\mud}[\varphi_\theta(\bm{x})]  \right\}.
\end{align}
Ideally, when updating $G_\theta$, one should follow the minimization \eqref{min phi_theta}, which takes into account the dependence on $\theta$ of the optimal discriminator $\varphi_\theta$. The generator update in WGAN, however, assumes that $\varphi$ is fixed %(independent of $\theta$) 
and performs %focuses on the simpler problem
\begin{align} \label{min phi}
&\min_{\theta}  \left\{ \mathbb{E}_{\bm{z}\sim \mu^{\bm{z}}} [\varphi(G_{\theta}(\bm{z}))] -  \mathbb{E}_{\bm{x}\sim\mud}[\varphi(\bm{x})]\right\} \\
=\ &\min_\theta \mathbb{E}_{\bm{z}\sim \mu^{\bm{z}}} [\varphi(G_{\theta}(\bm{z}))],\notag
\end{align}
which results in the update rule \eqref{Eq: WGAN Update} for $\theta$ in WGAN. Following \eqref{min phi}, but not \eqref{min phi_theta}, inevitably prevents WGAN from accurately minimizing $W_{1}\big((G_{\theta})_{\#}\mu^{\bm{z}}, \mud\big)$---after all, it is \eqref{min phi_theta} %, which considers the $\theta$-dependence of the discriminator $\varphi$, 
that equals $\min_{\theta} W_{1}\big((G_{\theta})_{\#}\mu^{\bm{z}}, \mud\big)$ theoretically. 

This issue may be exacerbated by performing \eqref{Eq: WGAN Update} multiple times with the same minibatch $\{ z_{i} \}$ (i.e., enforcing persistent training in WGAN as indicated in Remark~\ref{rem:W1-FE and WGAN}). Indeed, in so doing, one may better approximate the value of \eqref{min phi}. But since the values of \eqref{min phi} and \eqref{min phi_theta} are in general distinct, getting closer to \eqref{min phi} may amount to moving further away from \eqref{min phi_theta}, thereby deviating further from $\min_{\theta} W_{1}\big((G_{\theta})_{\#}\mu^{\bm{z}}, \mud\big)$. As a numerical experiment in Section~\ref{sec:examples} shows, adding persistent training destabilizes WGAN and worsens training results; see Figure~\ref{fig: CIFAR-10 LP Quantitative Results}. 
%By writing $\varphi_\theta = \varphi_{(G_{\theta})_{\#}\mu^{\bm{z}}}^{\mud}$, the Kantorovich potential in Definition~\ref{def:KP} (which attains the inside maximum in \eqref{Eq: True Problem}),  

Such an issue is averted in our ODE setup. Theoretically, we never consider \eqref{min phi} but approach the original problem $\min_{\theta} W_{1}\big((G_{\theta})_{\#}\mu^{\bm{z}}, \mud\big)$ head-on, by deriving the ``gradient'' of $\mu\mapsto W_1(\mu,\mud)$ in $\Pc_1(\R^d)$ and the corresponding gradient-decent ODE \eqref{Eq: ODE}. Algorithmically, 
the discriminator $\varphi$ gives rise to points $\{\zeta_i\}$ that follow the new distribution induced by the ODE at the next time step, and the generator $G_\theta$ serves to recover the points $\{\zeta_i\}$, so as to properly represent the new distribution. As explained in the second and third paragraphs below Algorithm~\ref{Alg: W1-FE}, persistent training can help $G_\theta$ better recover the points $\{\zeta_i\}$, such that the ODE is more closely followed. In other words, persistent training in W1-FE (i.e., Algorithm~\ref{Alg: W1-FE}) allows for more accurate simulations of ODE \eqref{Eq: ODE}, which by construction reduces $W_{1}\big((G_{\theta})_{\#}\mu^{\bm{z}}, \mud\big)$ via gradient descent in $\Pc_1(\R^d)$. On the other hand, persistent training in WGAN (as indicated in Remark~\ref{rem:W1-FE and WGAN}) may produce a more accurate solution to \eqref{min phi}, which however does not correspond to the minimization of $W_{1}\big((G_{\theta})_{\#}\mu^{\bm{z}}, \mud\big)$.

\section{Numerical Experiments}\label{sec:examples}
This section contains three training experiments with datasets from low to high dimensions. In each experiment, we carry out the training task using W1-FE-LP with diverse persistency levels $K=1,3,5,10$. The case $K=1$ can be viewed as the baseline, as Proposition~\ref{Prop: W1-GANs simulate our process} indicates that it is equivalent to the refined WGAN algorithm in \citet{Petzka18} (i.e., W1-LP), which is arguably one of the most well-performing and stable WGAN algorithms.

First, we consider learning a two-dimensional mixture of Gaussians arranged on a circle from an initially given Gaussian distribution \citep{Metz17}. Figure~\ref{fig: 2D Qualitative Results} shows the qualitative evolution of the models, while Figure~\ref{fig: 2D Quantitative Results} presents the actual $W_1$ losses. In Figure~\ref{fig: 2D Quantitative Results}, W1-FE-LP with $K=3$ and $K=5$ converge to a similar loss level as the baseline $K=1$ case (i.e., W1-LP), but achieves it much faster than W1-LP in both number of epochs\footnote{One ``training epoch'' refers to learning one Euler step in \eqref{Discrete Time W1 Transport Process}, i.e., one iteration of the loop in Algorithm~\ref{Alg: W1-FE}.} and wallclock time. These gains are partially lost with $K=10$, possibly due to overfitting. 
%\citet{Metz17} introduced a two dimensional mixture of Gaussians to demonstrate the superiority of unrolled GANs. In this section, we will use our algorithms to learn that dataset from a standard Gaussian distribution. We chose this dataset because other GAN models have had difficulty learning this mixture of Gaussians (see \citet[Figure 2]{Arjovsky17}). As we shall see, persistent training may considerably accelerate training time.
%In particular, we perform W1-FE-LP with difference persistency levels $K=1,3,5,10$. The case $K=1$ can be viewed as our baseline, for by Proposition~\ref{Prop: W1-GANs simulate our process}, the algorithm is equivalent to W1-LP. The results are shown in Figure~\ref{fig: 2D Qualitative Results} and Figure~\ref{fig: 2D Quantitative Results}. 
The computation in Figures~\ref{fig: 2D Qualitative Results} and \ref{fig: 2D Quantitative Results} takes $10$ discriminator updates per training epoch,  $\gamma_{g}  = 10^{-4}$, minibatches of size $m=512$, and uses a three-layer perceptron for the generator and discriminator, where each hidden layer contains $128$ neurons. All neural networks are trained using the Adam stochastic gradient update rule \citep{Kingma17}. We let $\epsilon = 1$, for $\gamma_{g}$ is already small and thus controls any possible overshooting from backpropagation.

\begin{figure}[h]
\centering
\includegraphics[width=8cm]{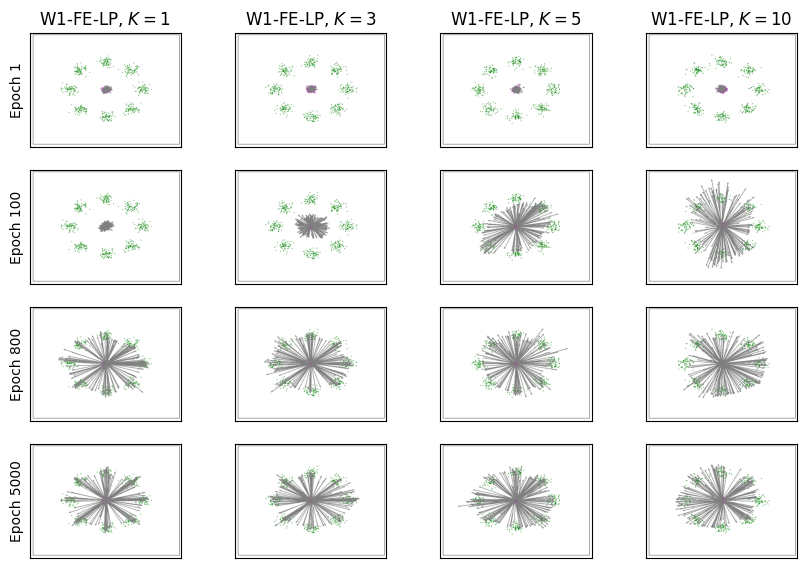}
\caption{\label{fig: 2D Qualitative Results} Qualitative evolution of learning process. A sample from the target distribution is given in green, a sample from the initial distribution is in magenta, and the transport rays by the generator are given in the grey arrows. The generated samples lie at the head of each grey arrow.}
\end{figure}

\begin{figure}[h]
\centering
\includegraphics[width=8cm]{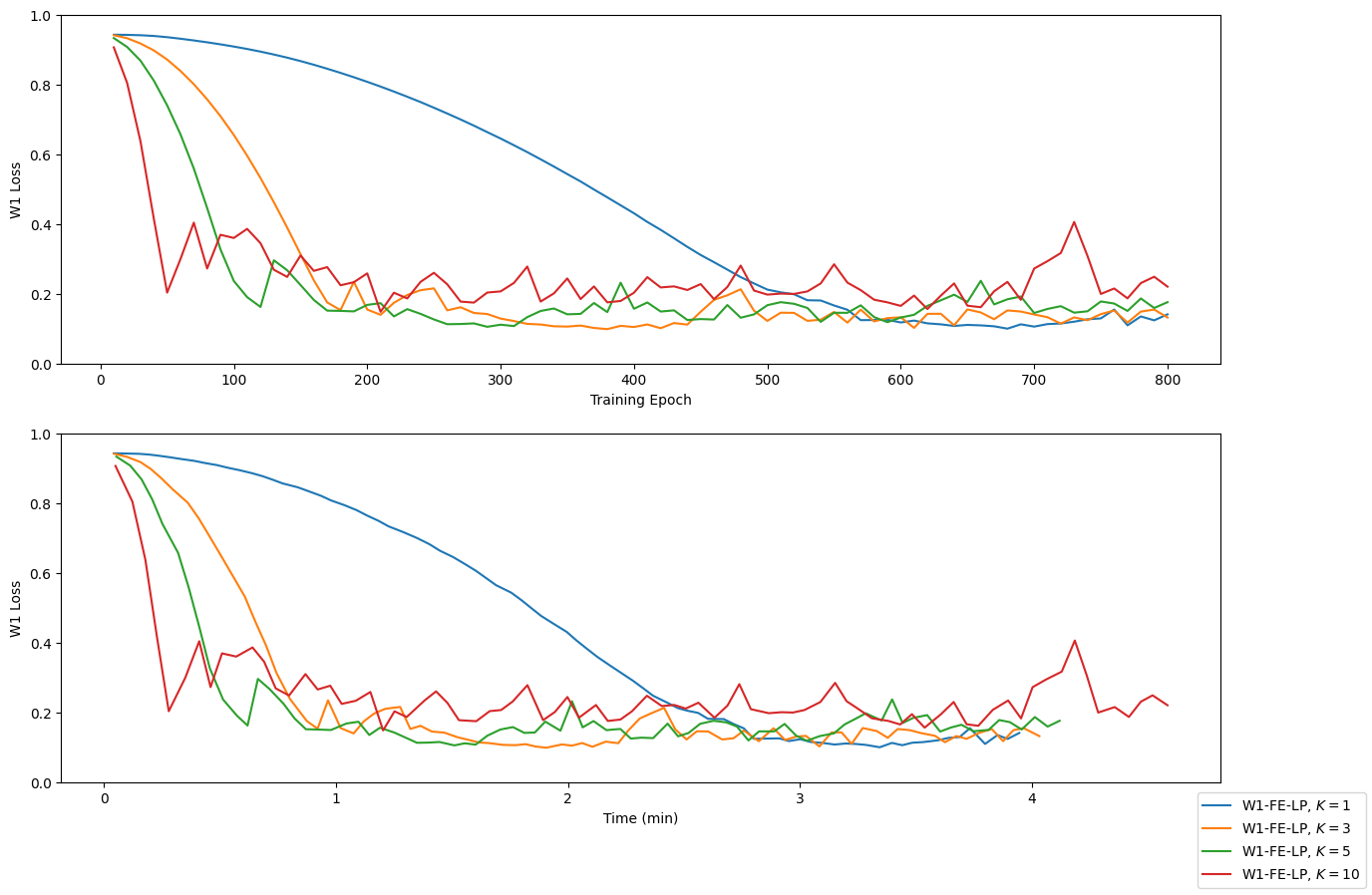}
\caption{\label{fig: 2D Quantitative Results} $W_1$ loss of W1-FE-LP with persistency levels $K=1,3,5,10$ against training epoch (top) and wallclock time (bottom), respectively.}
\end{figure}

Next, we consider domain adaptation from the USPS dataset \citep{Hull94} to the MNIST dataset \citep{Deng12} and evaluate the performance of our algorithms every 100 training epochs using a 1-nearest neighbor (1-NN) classifier. This is the same performance metric as in \citet{Seguy18}, although it was used there only once at the end of training. Figure~\ref{fig: High-D Quantitative Results} displays the results. With the persistency level $K = 3$, W1-FE-LP converges significantly faster and consistently achieves a higher accuracy rate than the baseline case $K=1$ (i.e., W1-LP). Indeed, it takes W1-LP 6000 epochs to attain its ultimate accuracy rate, which is achieved by W1-FE-LP with $K = 3$ by epoch 2000; W1-FE-LP with $K = 3$ continues to improve after epoch 2000, yielding the best accuracy rate among all the models. Raising persistency level to $K = 5$ and $K=10$ further accelerates the convergence before epoch 2000, but the ultimate accuracy rate achieved worsens slightly, which might result from overfitting.
%we apply our framework to the unsupervised domain adaptation problem as considered \citet[Section 5.2]{Seguy18}. In particular, we use our W1-FE-LP algorithms with varying persistency to solve the domain adaptation problem from the USPS \citet{Hull94} dataset to the MNIST \citet{Deng12} dataset. By virtue of Proposition~\ref{Prop: W1-GANs simulate our process}, we can effectively treat W1-FE-LP with $K=1$ as W1-LP. Hence, we use that as our baseline method. Adaptation performance is evaluated every $100$ training epochs using a $1$-nearest neighbor ($1-$NN) classifier. This is the same performance metric used in \citet{Seguy18}; the difference being that the authors of \citet{Seguy18} only evaluated their models at the end of training, whereas we evaluate each model throughout training. 
%The results are displayed in Figure~\ref{fig: High-D Quantitative Results}. From this figure, we can observe that persistent training yields far faster convergence for the algorithms than without it. While it may be too close to be conclusive, we also see that W1-FE-LP with $K=3$ obtains the best overall accuracy, whereas the $K=5$ experiment performed slightly worse. In contrast, the $K=10$ experiment performed the absolute worst among all model tested. This is likely a result of overfitting, where $K=3$ may be a ``sweet spot" for the USPS to MNIST adaptation problem. 
The computation in Figure~\ref{fig: High-D Quantitative Results} takes $\gamma_{g} = 10^{-4}$, $\epsilon = 1$, minibatches of size $m=64$, and $5$ discriminator updates per training epoch. 
%Every experiment was run on the T4 GPU available via Google Colab.

\begin{figure}[h]
\centering
\includegraphics[width=8cm]{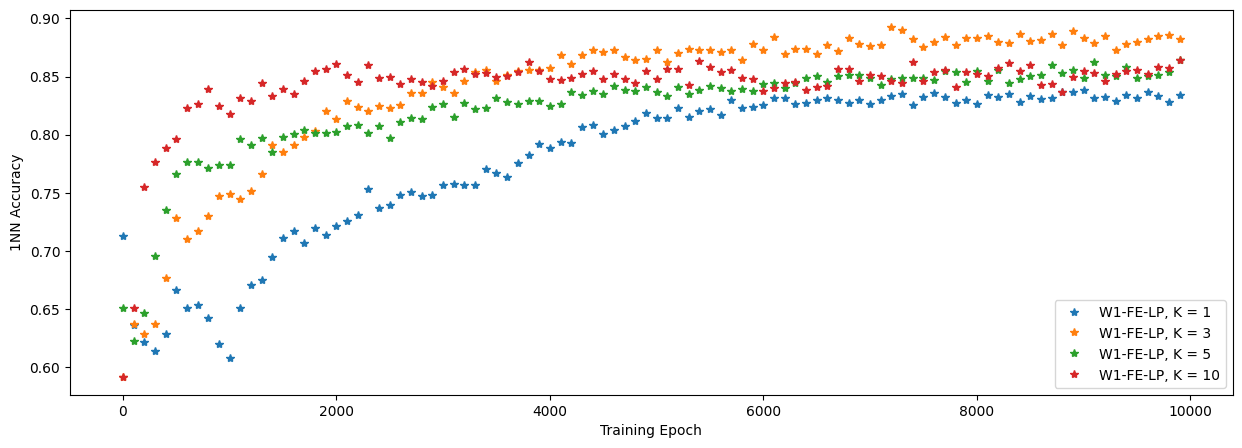}
\caption{\label{fig: High-D Quantitative Results} $1$-NN classifier accuracy against training epoch for domain adaptation from USPS to MNIST datasets.}
\end{figure}

Finally, we train our algorithms on the CIFAR-10 dataset. The prior distribution (i.e., the input $z$ into the generator in Algorithm~\ref{Alg: W1-FE}) is a $100$-dimensional standard Gaussian and we transform it through a multi-layer convolutional neural network into an image of dimension $3\times 32 \times 32$. %Both the generator and discriminator are multi-layer convolutional neural networks. 
The performance is evaluated by the Fr\'{e}chet inception distance (FID) \citep{Heusel18}, a common criterion for training quality of high-dimensional images. The results are displayed in Figure~\ref{fig: CIFAR-10 LP Quantitative Results} (top), where the FID is calculated using a pre-trained ResNet18 model on CIFAR-10, publicly available on the Github repository \citet{Phan_GH}. With a larger persistency level $K$, W1-FE-LP converges faster and consistently achieves a lower FID than the baseline model $K=1$ (i.e., W1-LP). In particular, it takes W1-LP 30000 epochs to achieve an FID of around 8. This same level of quality is achieved by W1-FE-LP with $K=10$ by epoch 10000. A visual inspection of Figure~\ref{fig: CIFAR-10 LP Qualitative Results} shows that images generated by epoch $500$ are far clearer under $K=10$ than under smaller $K$ values. Figure~\ref{fig: CIFAR-10 LP Quantitative Results} (top) also shows that W1-FE-LP with $K=5$ achieves the smallest FID by epoch $30000$. Raising the persistency level to $K=10$ further accelerates the convergence before epoch 20000, but the ultimate FID achieved worsens slightly, possibly due to overfitting.

We also perform experiments on CIFAR-10 under W1-LP with persistent training included, as indicated by Remark~\ref{rem:W1-FE and WGAN}. Figure~\ref{fig: CIFAR-10 LP Quantitative Results} (bottom) shows that the FID diverges as  the persistency level $K\in\N$ increases and the lowest FID is achieved under $K=1$ (i.e., no persistent training). This confirms the  discussion below Remark~\ref{rem:W1-FE and WGAN}, which asserts that persistent training can hurt the performance of WGAN. A visual inspection of samples (Section \ref{subsec:WGAN persistency}) also shows the worsened performance qualitatively. 
The computation in Figures \ref{fig: CIFAR-10 LP Quantitative Results} and \ref{fig: CIFAR-10 LP Qualitative Results} takes $\gamma_{g} = 10^{-4}$, time step $\epsilon = 1$, minibatches of size $m=64$, and $5$ discriminator updates per training epoch.

\begin{figure}[h]
\centering
\includegraphics[width=8cm]{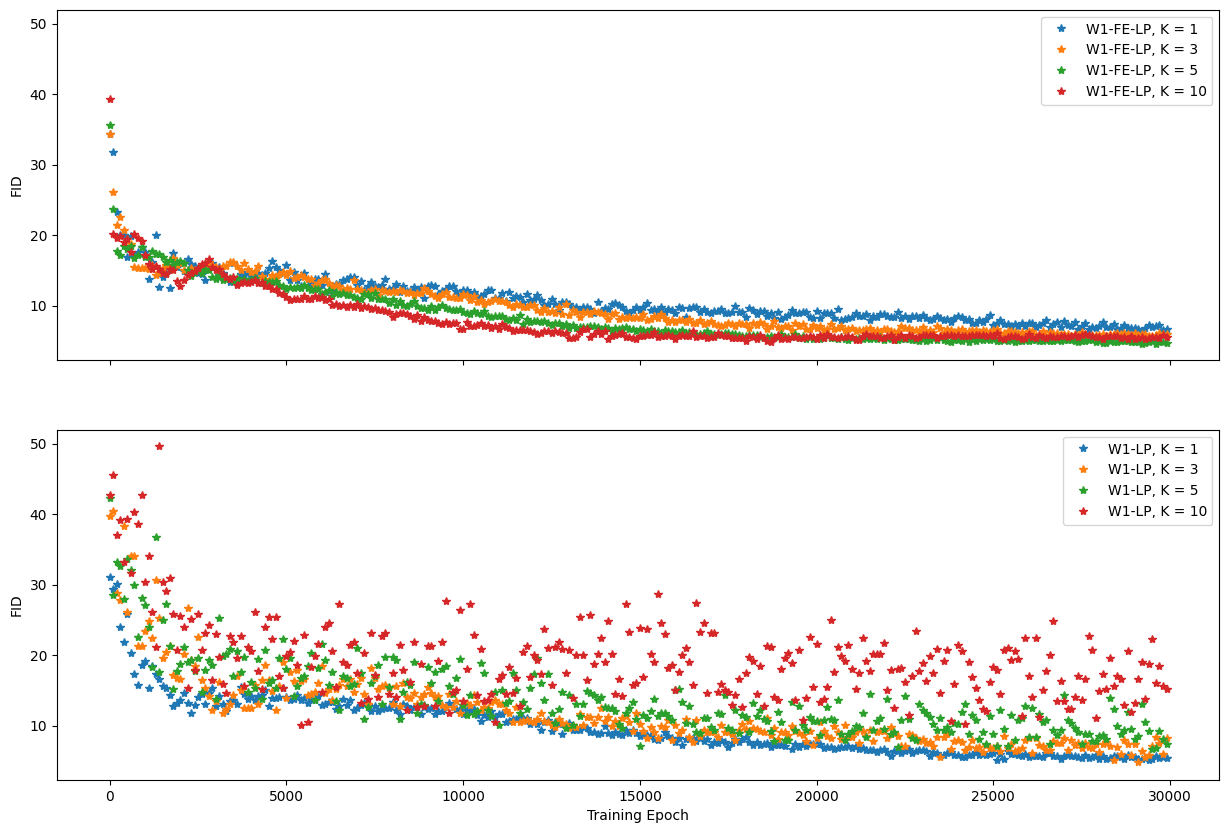}
\caption{\label{fig: CIFAR-10 LP Quantitative Results} FID against training epoch for various W1-FE-LP (top) and W1-LP (bottom) models on generating CIFAR-10 images. This demonstrates how persistent training may destabilize the training procedure in other WGANs.}
\end{figure}

\begin{figure}[h]
\centering
\includegraphics[width=8cm]{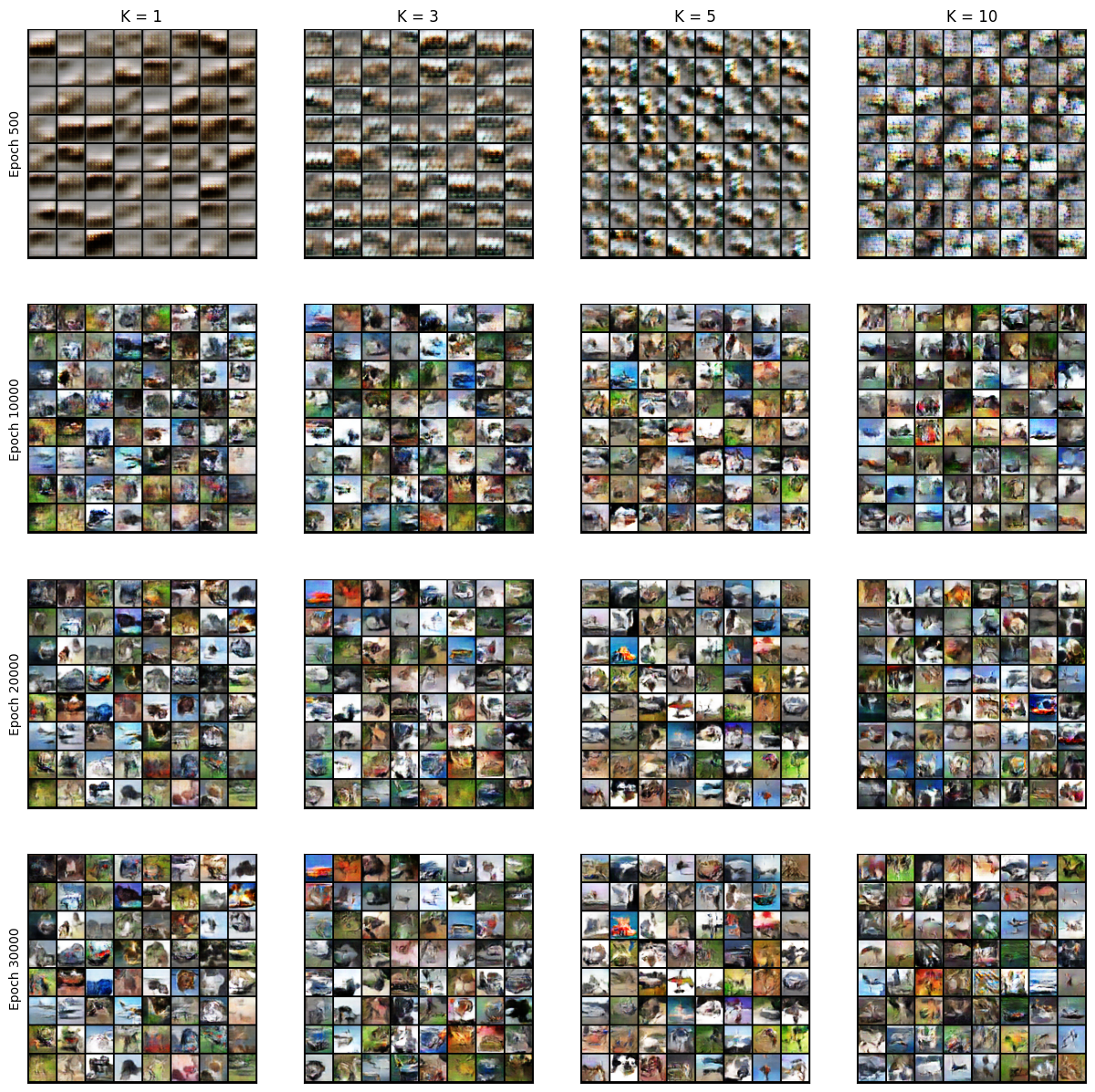}
\caption{\label{fig: CIFAR-10 LP Qualitative Results} Uncurated samples from various W1-FE-LP models across training.}
\end{figure}

%%%%%%%%%%%%%%%%%%%%%%%%%%%%%%
%%%%%%%%%%%%%%%%%%%%%%%%%%%%%%

\section{Limitations}\label{sec:limitations}
While we have made important progress on the theoretical development, several questions remain open. Recall that Algorithm~\ref{Alg: W1-FE} builds upon \eqref{Discrete Time W1 Transport Process}, a discretization of the gradient-flow ODE \eqref{Eq: ODE}. While this discretization has a well-defined limit $\mu^*(t)$ in continuous time (Theorem~\ref{Th: Convergence of interpolating measure curve}), whether $\mu^*(t)$ truly solves ODE \eqref{Eq: ODE} is left unanswered. It is also unclear if $\mu^*(t)$ will ultimately converge to the data distribution $\mud$, although it is intuitively expected by our ``gradient descent'' idea. To fill these gaps, one wishes to show that (i) there exists a (unique) solution $Y$ to ODE \eqref{Eq: ODE}, (ii) the law of $Y_t$ coincides with $\mu^*$, i.e., $\mu^{Y_t}=\mu^*(t)$ for all $t\ge 0$, and (iii) $\mu^{Y_t}$ ultimately converges to $\mud$, i.e., $W_1(\mu^{Y_t},\mud)\to 0$ as $t\to\infty$. The challenge here is twofold. As the coefficient $(\mu,x)\mapsto \nabla \varphi_{\mu}^{\mud}(x)$ of \eqref{Eq: ODE} is not continuous in general, standard existence results for distribution-dependent (stochastic) differential equations (i.e., McKean-Vlasov equations) cannot be easily applied. Also, when analyzing the flow of measures $\{\mu^{Y_t}\}_{t\ge 0}$ in $\Pc_1(\R^d)$ through the continuity equation (or, Fokker-Planck equation) associated with \eqref{Eq: ODE}, standard results in \citet{Ambrosio08} are not applicable, as they cover the case $\Pc_p(\R^d)$ for all $p>1$ but excludes  $\Pc_1(\R^d)$. 

 %the lack of strict convexity in the cost function $|\cdot|$ makes it difficult to obtain analogous results in $\Pc_{1}(\Xc)$. 
%In light of Proposition~\ref{Prop: W1-GANs simulate our process}, an easier first step could be showing that WGANs converge to the data distribution. However, the authors are currently unaware of any kind of proof for this result.

Numerically, while we showed that persistent training can markedly improve training results in several experiments, it is not without restrictions. Recall that SGD is performed $K\in\N$ consecutive times with the same minibatch $\{\zeta_i\}$ in Algorithm~\ref{Alg: W1-FE}, with $\zeta_{i} = G_{\theta}(z_{i}) - \epsilon \nabla \varphi(G_{\theta}(z_{i}))$. There are two issues one has to confront. First, any inaccuracy in the estimation of the Kantorovich potential $\varphi$ will be amplified by persistent training. As a larger $K\in\N$ demands that our algorithm more closely fits the data points $\{\zeta_i\}$, even when $\{\zeta_i\}$ are of low quality due to the inaccuracy of $\varphi$, the issue of ``garbage in, garbage out'' will be exacerbated. Second, even if $\varphi$ is perfectly estimated, such that $\{\zeta_i\}$ are of high quality, an excessive $K\in\N$ will lead to overfitting. 

The first issue can be mitigated by better estimating the Kantorovich potential $\varphi$. Indeed, experiments in Section~\ref{sec:examples} are run using W1-FE-LP, not W1-FE-GP, because the former is known to estimate $\varphi$ more accurately than the latter \citep{Petzka18}; recall the distinction between the two algorithms in the paragraph below Remark~\ref{rem:robustness}. In Appendix~\ref{subsec:GP}, we run the first experiment in Section~\ref{sec:examples} again using W1-FE-GP and compare the results with those under W1-FE-LP previously presented. It shows that raising persistency levels results in significantly more severe instability under W1-FE-GP than under W1-FE-LP. On the other hand, to mitigate overfitting, we suggest finding a suitable persistency level through careful numerical investigation. The experiments in Section~\ref{sec:examples} all indicate a threshold of $K\in\N$ beyond which the performance starts to deteriorate (i.e., $K=5$, $K=3$, and $K=5$ in the first, second, and third experiments, respectively). Taking $K\in\N$ to be at such a threshold (but not beyond it) can likely balance the benefits of persistent training against overfitting.

\section{Conclusion}\label{sec:conclusion}
By performing ``gradient descent'' in the space $\Pc_{1}(\R^d)$, we introduce a distribution-dependent ODE for the purpose of generative modeling. A forward Euler discretization of the ODE converges to a curve of probability measures, suggesting that a numerical implementation of the discretization is stable for small time steps. This inspires a class of new algorithms (called W1-FE) that naturally involves persistent training. If we (artificially) turn off persistent training, our algorithms recover WGAN algorithms. By increasing the level of persistent training suitably, our algorithms outperform WGAN algorithms in numerical examples. 
  %Remark~\ref{Rem: Gradient Flow and OT} even suggests an intimate connection between our ODE and the optimal transport framework in $\Pc_{1}(\Xc)$, although we make no further claims in this direction. 
%We then showed that many existing Wasserstein GANs belong to our proposed class of algorithms. Furthermore, only the ODE discretization allows for persistent training, which we showed can accelerate training time on a few numerical examples. 

%\section{Acknowledgements}
%We thank Google for their freely available Google Colab, which we used %for all simulations.

\section{Impact Statement}
This paper presents work whose goal is to advance the field of Machine Learning. There are many potential societal consequences of our work, none of which we feel must be specifically highlighted here.

\bibliographystyle{icml2025}
\bibliography{bibliography.bib}

\appendix
\section{Theoretical Results}
\subsection{Convexity of $W_{1}(\cdot,\mu_{d})$}\label{subsec:convexity of J}
\begin{proposition}
The function $J:\Pc_1(\R^d)\to\R$ in \eqref{J} is convex. That is, for any $\mu, \nu \in \mathcal{P}_{1}(\R^d)$ and $\lambda \in (0,1)$, we have $J((1-\lambda)\mu + \lambda \nu) \leq \lambda J(\mu) + (1-\lambda) J(\nu)$.
\end{proposition}

\begin{proof}
For any $\mu, \nu \in \mathcal{P}_{1}(\R^d)$ and $\lambda \in (0,1)$, let $\varphi_{\lambda}$ denote a Kantorovich potential from $(1-\lambda)\mu + \lambda \nu\in\Pc_1(\R^d)$ to $\mud$. By \eqref{J}, \eqref{Eq: W1 Duality Formula}, and Definition~\ref{def:KP},  
    \begin{equation*}
        \begin{aligned}
            J((1-\lambda)\mu& + \lambda \nu) = W_1((1-\lambda)\mu + \lambda \nu,\mud) \\
            &= \int_{\mathbb{R}^{d}} \varphi_{\lambda} \, d((1-\lambda)\mu + \lambda \nu) - \int_{\mathbb{R}^{d}} \varphi_{\lambda} \, d\mu_{d}\\
            &= (1-\lambda)\int_{\mathbb{R}^{d}} \varphi_{\lambda} \, d(\mu- \mu_{d}) + \lambda \int_{\mathbb{R}^{d}} \varphi_{\lambda} \, d(\nu - \mu_{d}) \\
            &\leq (1-\lambda) \sup_{||\varphi||_{\operatorname{Lip}} \leq 1} \left \{ \int_{\mathbb{R}^{d}} \varphi \, d(\mu-\mu_{d}) \right \} \\
            &\qquad+ \lambda \sup_{||\varphi||_{\operatorname{Lip}} \leq 1} \left \{ \int_{\mathbb{R}^{d}}  \varphi \, d(\nu - \mu_{d}) \right\} \\
            &= (1-\lambda) J(\mu) + \lambda J(\nu),
        \end{aligned}
    \end{equation*}
    where the last equality follows again from \eqref{J} and \eqref{Eq: W1 Duality Formula}. 
\end{proof}

\begin{remark}
In most cases, the inequality in the proof above is strict, as it is in general unlikely that $\varphi_{\lambda}$ also attains both of the two suprema.
\end{remark}

\subsection{Proof of Proposition~\ref{Prop: LFD of W1 distance}}\label{subsec:proof of Prop: LFD of W1 distance}
\begin{proof}
Fix $\mu,\nu\in\Pc_1(\R^d)$. For any $\epsilon \in (0,1)$, note that $\mu + \epsilon(\nu-\mu) = (1-\epsilon)\mu+\epsilon\nu$ remains in $\Pc_1(\R^d)$. By \eqref{J} and Definition~\ref{def:KP}, 
\begin{align}
J(\mu) &= W_1(\mu,\mud) =  \int_{\R^d} \varphi_{\mu}^{\mu_{d}} \, d(\mu - \mud),\label{equ1}
\end{align}
and
\begin{equation}
    \begin{aligned}
        J(\mu + \epsilon(\nu-\mu)) &= W_1(\mu + \epsilon(\nu-\mu),\mud) \\
        &=  \int_{\R^d} \varphi_{\mu+\epsilon(\nu-\mu)}^{\mud} \, d(\mu + \epsilon(\nu-\mu) - \mud).\label{equ2}
    \end{aligned}
\end{equation}
On the other hand, by \eqref{J}, the duality formula \eqref{Eq: W1 Duality Formula}, and the fact that $\varphi_{\mu}^{\mud},  \varphi_{\mu+\epsilon(\nu-\mu)}^{\mud}:\R^d\to\R$ are $1$-Lipschitz functions, we obtain the inequalities
\begin{equation}
    \begin{aligned}
        J(\mu) &= W_1(\mu,\mud) \\
        &\ge   \int_{\R^d} \varphi_{\mu+\epsilon(\nu-\mu)}^{\mud} \, d(\mu - \mud),\label{inequ1}
    \end{aligned}
\end{equation}
and 
\begin{equation}
    \begin{aligned}
        J(\mu + \epsilon(\nu-\mu))&= W_1(\mu + \epsilon(\nu-\mu),\mud) \\
        &\geq \int_{\R^d} \varphi_{\mu}^{\mud} \, d(\mu + \epsilon(\nu-\mu) - \mud).\label{inequ2}
    \end{aligned}
\end{equation}
It follows from \eqref{inequ2} and \eqref{equ1} that
    \begin{equation*}
        \begin{aligned}
            &J(\mu + \epsilon(\nu-\mu)) - J(\mu) \\
            &\qquad\geq \int_{\R^d} \varphi_{\mu}^{\mud} \, d(\mu + \epsilon(\nu-\mu) - \mud)  \\
            &\qquad \qquad - \int_{\R^d} \varphi_{\mu}^{\mud} \, d(\mu - \mud) \\
            &\qquad = \epsilon \int_{\R^d}  \varphi_{\mu}^{\mud} \, d(\nu - \mu), 
        \end{aligned}
  \end{equation*}
while \eqref{equ2} and \eqref{inequ1} imply
    \begin{equation*}
        \begin{aligned}
            &J(\mu + \epsilon(\nu-\mu)) -J(\mu) \\
            &\leq\int_{\R^d} \varphi_{\mu + \epsilon(\nu-\mu)}^{\mud} \, d(\mu + \epsilon(\nu-\mu) - \mud) \\
            &\qquad - \int_{\R^d} \varphi_{\mu + \epsilon(\nu-\mu)}^{\mud} \, d(\mu - \mu_{d})\\
            &= \epsilon \int_{\R^d} \varphi_{\mu + \epsilon(\nu-\mu)}^{\mud} \, d(\nu - \mu).
        \end{aligned}
    \end{equation*}
Putting the above two inequalities together, we see that
    \begin{equation*}
        \begin{aligned}
            &\int_{\R^d} \varphi_{\mu}^{\mud} \, d(\nu - \mu)\\
            &\qquad \leq \frac{J(\mu+\epsilon(\nu-\mu)) - J(\mu)}{\epsilon} \\
            &\qquad \leq \int_{\R^d} \varphi_{\mu+\epsilon(\nu-\mu)}^{\mud} \, d(\nu-\mu).
        \end{aligned}
    \end{equation*}
As $\epsilon \rightarrow 0^{+}$, since $\varphi_{\mu+\epsilon(\nu-\mu)}^{\mu_{d}}$ converges uniformly to $\varphi_{\mu}^{\mu_{d}}$ on compacts of $\R^d$ \citep[Theorem 1.52]{Santambrogio15}, the right-hand side above tends to $\int_{\R^d} \varphi_{\mu}^{\mud} \, d(\nu - \mu)$. This then implies
    \begin{equation*}
        \lim_{\epsilon \rightarrow 0^{+}} \frac{J(\mu + \epsilon(\nu - \mu)) - J(\mu)}{\epsilon} = \int_{\R^d} \varphi_{\mu}^{\mud} \, d(\nu - \mu),
    \end{equation*}
i.e., $\varphi_{\mu}^{\mud}$ is a linear functional derivative of $J$. 
\end{proof}

\subsection{A Refined Arzela-Ascoli Result}
The following is a transcription of \citet[Proposition 3.3.1]{Ambrosio08} in our specific setting, where we consider the metric space $\mathcal{P}_{1}(\mathbb{R}^{d})$ with the natural topology induced by the $W_1$ distance.

\begin{proposition} \label{Prop: Arzela-Ascoli}
Fix $T>0$ and let $K \subseteq \mathcal{P}_{1}(\mathbb{R}^{d})$ be compact in $\mathcal{P}_{1}(\mathbb{R}^{d})$ under the topology induced by the $W_1$ distance. For any sequence $\{g_n\}_{n\in\N}$ of curves $g_{n}:[0,T] \rightarrow \mathcal{P}_{1}(\mathbb{R}^{d})$ such that
        \begin{align}
            &g_{n}(t) \in K, \quad \forall n \in \mathbb{N},\ t \in [0,T],\label{in K}\\
            &\limsup_{n \rightarrow \infty} W_1(g_{n}(s), g_{n}(t)) \leq |s-t|, \, \forall s,t \in [0,T],\label{equicontinuous}
        \end{align}
    %\begin{eqnarray}
       % \limsup_{n \rightarrow \infty} d(g_{n}(s), g_{n}(t)) \leq |s-t| \quad \forall s,t \in [0,T].
    %\end{eqnarray}
there exist an increasing subsequence $k \rightarrow n(k)$ and a continuous $g:[0,T] \rightarrow \mathcal{P}_{1}(\mathbb{R}^{d})$ such that
    \begin{equation}
        W_{1}(g_{n(k)}(t), g(t)) \rightarrow 0 \quad \forall t \in [0,T].
    \end{equation}
   % and $g$ is $d$-continuous in $[0,T]$.
\end{proposition}

    %for a symmetric function $\omega:[0,T] \times [0,T] \rightarrow [0, + \infty)$, such that
    %\begin{equation} \label{Cond: omega zero at (t,t)}
       % \lim_{(s,t) \rightarrow (r,r)} \omega(s,t) = 0 \quad \forall r \in [0,T] \setminus \mathcal{N},
    %\end{equation}

\subsection{Proof of Theorem~\ref{Th: Convergence of interpolating measure curve}}\label{subsec:proof of convergence result}
\begin{proof}
Let $(\Omega, \mathcal F, \P)$ be the underlying probability space that supports all the random variables $\{Y_{n-1,\epsilon}:n\in\N, \epsilon>0\}$, defined as in \eqref{Discrete Time W1 Transport Process}. 
Fix any $T>0$. We will show that $\{\mu_\epsilon(t):\epsilon> 0, t\in[0,T]\}$ fulfills \eqref{in K} and \eqref{equicontinuous}. %First, we show that for any $\epsilon>0$, the set of measure $\{ \mu^{Y_{n, \epsilon}} \}_{n\in\N}$ is tight. 
For any fixed $t \in [0,T]$, there exists $n \in \mathbb{N}$ such that $t \in [(n-1)\epsilon, n\epsilon)$ and $\mu_{\epsilon}(t) = \mu^{Y_{n-1, \epsilon}}$. By \eqref{Discrete Time W1 Transport Process}, the random variable $Y_{n-1, \epsilon}$ takes the form 
\begin{eqnarray} \label{Eq: Decomposition of Y_n,eps}
        Y_{n-1,\epsilon} = Y_{0} - \epsilon \sum_{i = 0}^{n-2} \nabla  \varphi_{\mu^{Y_{i,\epsilon}}}^{\mu_{d}}(Y_{i, \epsilon}).
    \end{eqnarray}
As $|\nabla \varphi_\mu^\nu(x)|\le 1$ $\mathcal L^d$-a.e.\ on $\R^d$ for all $\mu,\nu\in\Pc_1(\R^d)$ (Remark~\ref{rem:nabla phi}), this implies 
\begin{equation}\label{bdd by Y_0+T}
|Y_{n-1,\epsilon}| \leq |Y_{0}| + (n-1) \epsilon\le  |Y_{0}| +t \le  |Y_{0}| +T\quad \hbox{a.s.,} 
\end{equation}
where the second inequality is due to $t \in [(n-1)\epsilon, n\epsilon)$. It follows that for all $t\in[0,T]$, 
 \begin{eqnarray}
        \begin{aligned}\label{E[|y|] bdd}
            \sup_{\epsilon >0} \int_{\mathbb{R}^{d}} |y| \, d\mu_{\epsilon}(t) &= \sup_{\epsilon > 0} \int_{\mathbb{R}^{d}} |y| \, d\mu^{Y_{n-1, \epsilon}} \\ & = \sup_{\epsilon> 0} \mathbb{E}^{\mathbb{P}}[|Y_{n-1,\epsilon}|] \\
            &\leq \mathbb{E}^\P [|Y_{0}|] + T< \infty,
        \end{aligned}
    \end{eqnarray}
By \eqref{E[|y|] bdd} and the fact that the function $\phi(y) := |y|$, $y \in \mathbb{R}^{d}$, has compact sublevels (i.e., the set $\{ y : |y| \leq c \}$  is compact in $\mathbb{R}^{d}$ for any $c \geq 0$), \citet[Remark 5.1.5]{Ambrosio08} asserts that the collection of measures $\{ \mu_{\epsilon}(t): \epsilon>0,t\in[0,T]\}$ is tight (i.e., precompact under the topology of weak convergence). To further prove that this collection of measures is precompact in $\Pc_1(\R^d)$, it suffices to show that the measures have uniformly integrable first moments, in view of \citet[Proposition 7.1.5]{Ambrosio08}. That is, we need to show that
\[
 \lim_{k\to\infty}   \sup_{\epsilon>0,t\in[0,T]}\int_{\mathbb{R}^{d} \setminus B_{k}(0)} |y| \, d\mu_{\epsilon}(t) =0,
\]
where $B_k(0)$ denotes the open ball centered at $0\in\R^d$ with radius $k>0$. For any fixed $t\in[0,T]$, by the same arguments above \eqref{Eq: Decomposition of Y_n,eps}, 
    \begin{align*}
        \int_{\mathbb{R}^{d} \setminus B_{k}(0)} |y| \, d\mu_{\epsilon}(t) &= \mathbb{E}^\P \left[ |Y_{n-1,\epsilon}|\ \mathbb{I}_{\mathbb{R}^{d} \setminus B_{k}(0)}(Y_{n-1,\epsilon}) \right]\\
         &= \E^\P\left[|Y_{n-1,\epsilon}(\omega)|\ \mathbb{I}_{\{|Y_{n-1,\epsilon}(\omega)| \ge k\}}(\omega)\right]\\
         &\le \E^\P\big[|Y_{0}(\omega)+T|\ \mathbb{I}_{\{|Y_{0}(\omega)+T|\ge k\}}(\omega)\big]
    \end{align*}
    where $\mathbb I$ denotes an indicator function and the inequality follows from \eqref{bdd by Y_0+T}. Hence,

    \begin{equation}
        \begin{aligned}
            &\sup_{\epsilon>0,t\in[0,T]}\int_{\mathbb{R}^{d} \setminus B_{k}(0)} |y| \, d\mu_{\epsilon}(t) \\
            &\qquad\leq \E^\P\big[|Y_{0}(\omega)+T|\ \mathbb{I}_{\{|Y_{0}(\omega)+T|\ge k\}}(\omega)\big] \\
            &\qquad \to 0, \; \hbox{as}\ k\to\infty,
        \end{aligned}
    \end{equation}
 %\[
 %\sup_{\epsilon>0,t\in[0,T]}\int_{\mathbb{R}^{d} \setminus B_{k}(0)} |y| \, d\mu_{\epsilon}(t)  \le \E^\P\big[|Y_{0}(\omega)+T|\ \mathbb{I}_{\{|Y_{0}(\omega)+T|\ge k\}}(\omega)\big]\to 0\quad \hbox{as}\ k\to\infty, 
 %\]
 where the convergence follows from $Y_0\in L^1(\P)$, thanks to $\mu^{Y_0}=\mu_0\in\Pc_1(\R^d)$.  We therefore conclude that $\{ \mu_{\epsilon}(t): \epsilon>0,t\in[0,T]\}$ is precompact  in $\Pc_1(\R^d)$ and thus fulfills \eqref{in K}. 

Next, consider any $s, t \in [0,T]$ with $s \neq t$. Without loss of generality, assume $s < t$. For any fixed $\epsilon >0$, there exist   $j, k \in \mathbb{N}$ with $j\le k$ such that
    \begin{eqnarray}\label{j and k}
        \begin{aligned}
            (j-1)\epsilon \le s < j \epsilon\ \ \hbox{and}\ \ \mu_{\epsilon}(s) = \mu^{Y_{j-1, \epsilon}}; \\
            (k-1)\epsilon \le t < k \epsilon \ \ \hbox{and}\ \ \mu_{\epsilon}(t) = \mu^{Y_{k-1, \epsilon}}.      
        \end{aligned}
    \end{eqnarray}
By \eqref{Discrete Time W1 Transport Process}, we have
       \begin{equation}
        Y_{k-1, \epsilon} = Y_{j-1, \epsilon} - \epsilon \sum_{i=1}^{k-j} \nabla \varphi_{\mu^{Y_{i-1,\epsilon}}}^{\mud}(Y_{i, \epsilon}).
    \end{equation}
It follows that
    \begin{eqnarray}
        \begin{aligned}
           W_{1}(\mu_{\epsilon}(s), \mu_{\epsilon}(t)) &= W_{1}(\mu^{Y_{j-1, \epsilon}}, \mu^{Y_{k-1, \epsilon}})\\
           &\leq \mathbb{E}^{\mathbb{P}} [|Y_{k-1, \epsilon} - Y_{j-1, \epsilon}|] \\
           &= \mathbb{E}^{\mathbb{P}} \left[ \left| \epsilon \sum_{i=1}^{k-j} \nabla \varphi_{\mu^{Y_{n-1,\epsilon}}}^{\mu_{d}}(Y_{i, \epsilon}) \right| \right]
           \\ 
           & \leq \epsilon (k-j) < |t - s| + \epsilon, 
        \end{aligned}
    \end{eqnarray}
where the second inequality follows again from $|\nabla \varphi_\mu^\nu(x)|\le 1$ $\mathcal L^d$-a.e.\ on $\R^d$ for all $\mu,\nu\in\Pc_1(\R^d)$ (Remark~\ref{rem:nabla phi}) and the third inequality is due to \eqref{j and k}. This immediately yields
\begin{equation}
        \limsup_{\epsilon \rightarrow 0}  W_{1}(\mu_{\epsilon}(s), \mu_{\epsilon}(t)) \leq |s-t|, \quad \forall s,t \in [0,T], 
    \end{equation}
 i.e., $\{ \mu_{\epsilon}(t): \epsilon>0,t\in[0,T]\}$ satisfies \eqref{equicontinuous}. 
    
Now, we can apply Proposition~\ref{Prop: Arzela-Ascoli} to obtain a subsequence $\{ \epsilon_{k} \}$ and a continuous curve $\mu^{*}_T(t):[0,T]\to\Pc_1(\R^d)$ such that $W_{1}(\mu_{\epsilon_{k}}(t),\mu^{*}_T(t))\to 0$ for all $t\in [0, T]$. By a diagonal argument, we can construct a continuous $\mu^*:[0,\infty)\to\Pc_1(\R^d)$ such that $W_{1}(\mu_{\epsilon_{k}}(t),\mu^{*}(t))\to 0$ for all $t\ge 0$, possibly along a further subsequence.   
\end{proof}

\section{More Experimental Results}
\subsection{Persistency on W1-FE-GP}\label{subsec:GP}
We run the first experiment in Section~\ref{sec:examples} again using W1-FE-GP. The results, along with those under W1-FE-LP in the main text, are shown in Figure~\ref{fig: W1-GP 2D Quantitative Results}. As we can see, raising persistency levels results in significantly more severe instability under W1-FE-GP than under W1-FE-LP. %This suggests that an accurate calculation of the Kantorovich potential is essential for improving generator training by increasing persistency.

\begin{figure}[h]
\centering
\includegraphics[width=8cm]{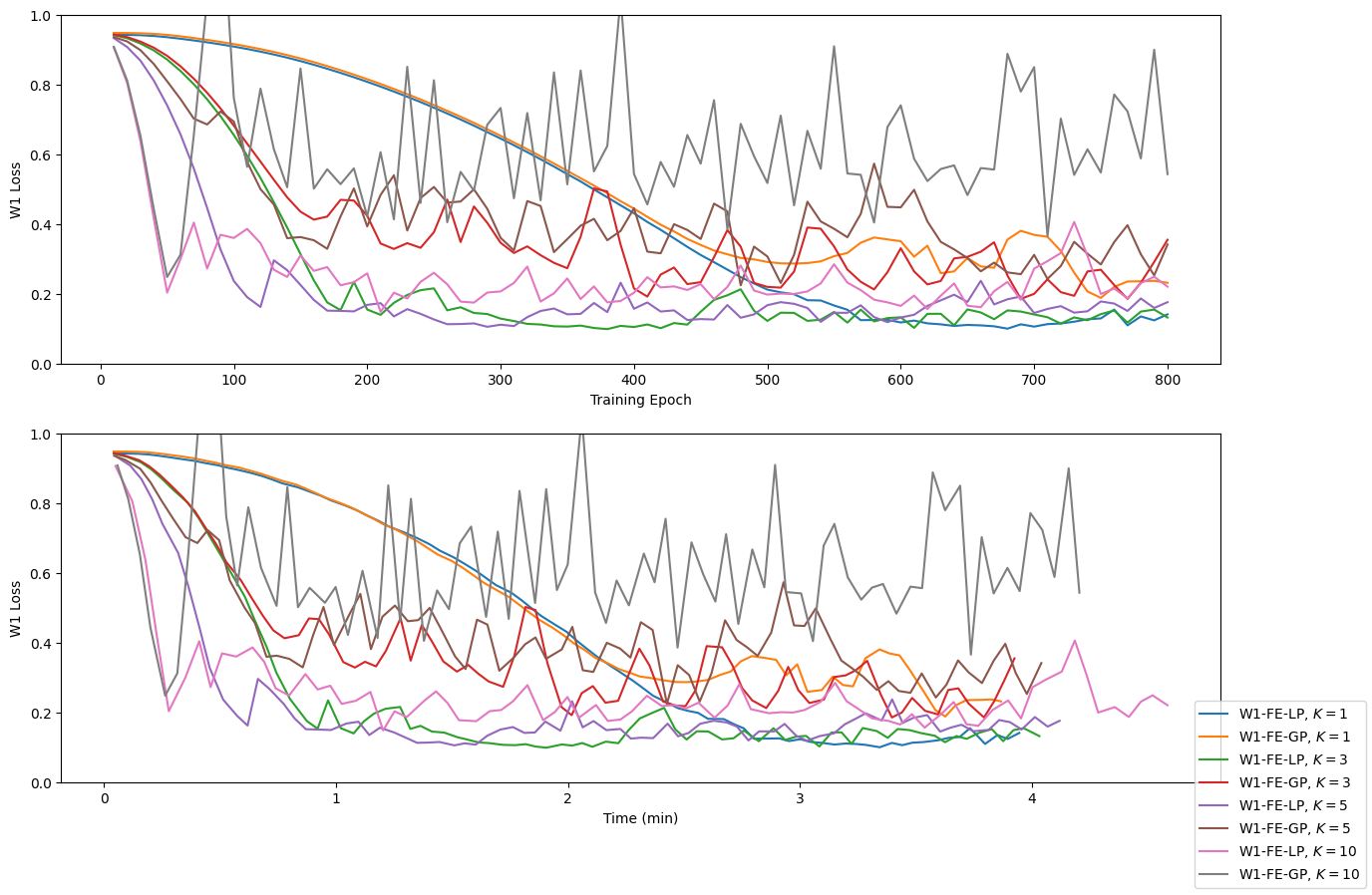}
\caption{\label{fig: W1-GP 2D Quantitative Results} 
$W_1$ loss of W1-FE-GP and W1-FE-LP with persistency levels $K=1,3,5,10$ against training epoch (left) and wallclock time (right), respectively.}
\end{figure}

\subsection{Persistency on WGANs}\label{subsec:WGAN persistency}
Persistent training on WGANs was shown to diverge on the CIFAR-10 dataset. We include a qualitative evolution of the training process to show how bad persistency is for WGANs.

\begin{figure}
    \centering
    \includegraphics[width=0.5\linewidth]{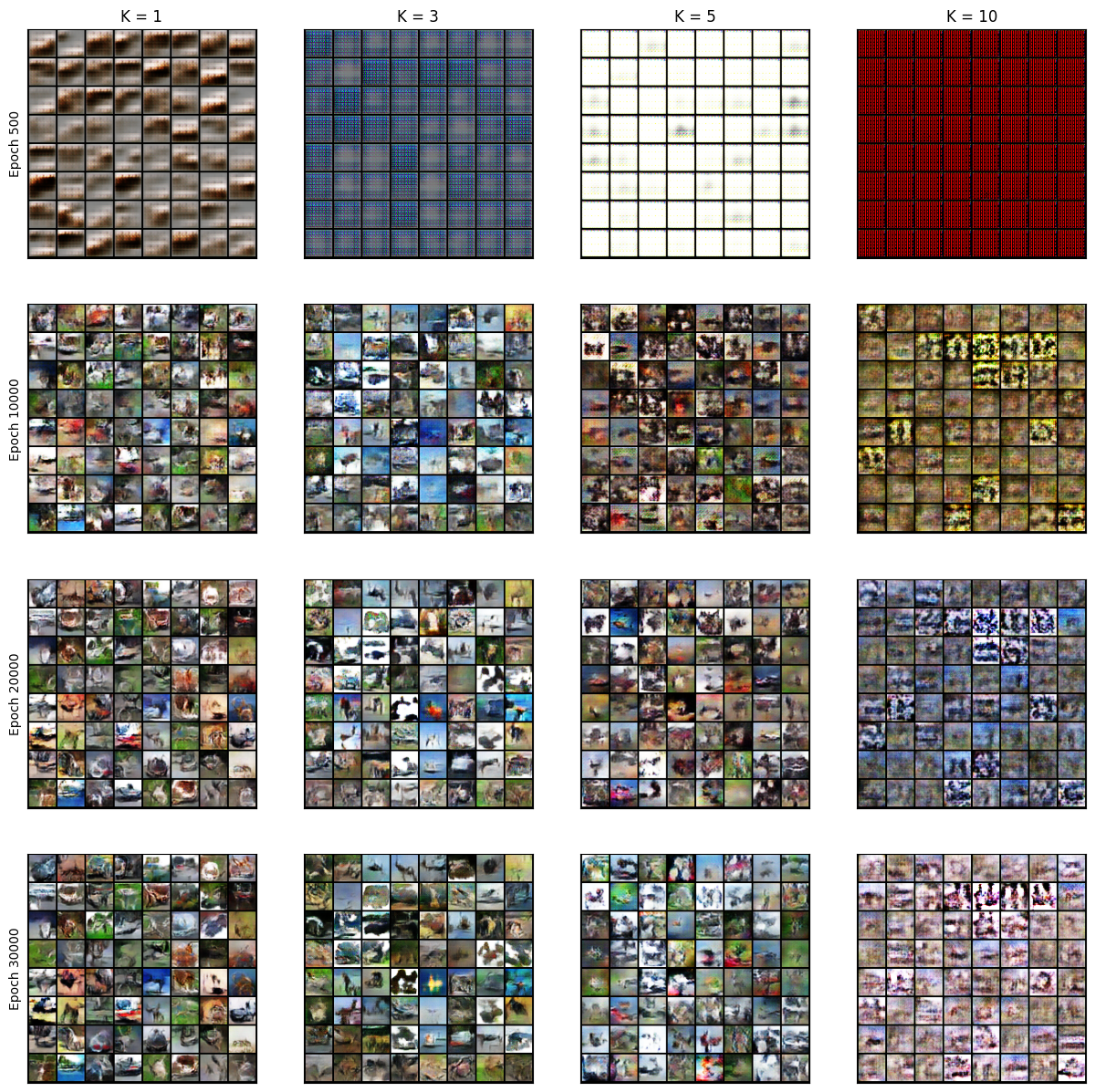}
    \caption{Uncurated samples from various persistent WGAN models across training. Increasing persistency generally results in worse performance.}
    \label{fig: Persistent training for WGANs.}
\end{figure}

\section{Using the code}
We built off of the software package developed for use in \cite{Leygonie19}. While we made substantial changes to the package for our own purposes, we do acknowledge that the package built by \cite{Leygonie19} made it substantially easier for us to implement our algorithm. The usage is almost identical to the original package's usage.

We recommend storing the code as either a zipped file or pulling directly from the GitHub repository. We also recommend using a Google Colab notebook as the virtual environment. Once the software package is loaded in the appropriate folder, one may reproduce the low dimensional experiments by running \texttt{main.py} inside \texttt{exp\textunderscore 2d}. The high dimensional experiments may be reproduced by running \texttt{main.py} inside \texttt{exp\textunderscore da}. 

If one uses Google Colab to run the experiments, then the default environment provided by the Google Colab Jupyter notebook in addition to the package Python Optimal Transport (POT) is required to run the software. To reproduce the plots, one needs the package \texttt{tensorboard}.

\end{document}